\documentclass{article} % For LaTeX2e
\usepackage{iclr2027_conference,times}

\usepackage{amsmath,amsfonts,bm}

\def\eqref#1{equation~\ref{#1}}
\def\1{\bm{1}}

\DeclareMathAlphabet{\mathsfit}{\encodingdefault}{\sfdefault}{m}{sl}
\SetMathAlphabet{\mathsfit}{bold}{\encodingdefault}{\sfdefault}{bx}{n}

\usepackage{url}

\usepackage[utf8]{inputenc} % allow utf-8 input
\usepackage[T1]{fontenc}    % use 8-bit T1 fonts
\usepackage[colorlinks=true, citecolor=blue]{hyperref}
\usepackage{booktabs}       % professional-quality tables
\usepackage{amsfonts}       % blackboard math symbols
\usepackage{nicefrac}       % compact symbols for 1/2, etc.
\usepackage{microtype}      % microtypography
\usepackage[table,dvipsnames]{xcolor}         % colors
\usepackage{graphicx}
\usepackage{subcaption}
\usepackage{enumitem}
\usepackage{amsmath}
\usepackage{multirow}
\usepackage{algorithm}
\usepackage{algorithmic}
\usepackage{wrapfig}
\usepackage{adjustbox}
\usepackage{amsthm}

\newcommand{\mystd}[1]{\footnotesize{\textpm{ #1}}}

\title{Hub for Outliers, Spokes for Inliers: Uniform Latent Space Construction for Dual-Mismatched Semi-Supervised Learning}

\author{Li Yuan\textsuperscript{1*}, Yaxin Hou\textsuperscript{1*}, Jiawei Tang\textsuperscript{1}, Yongbiao Gao\textsuperscript{2,3}, Yuheng Jia\textsuperscript{1,4,5$\dagger$} \\
\textsuperscript{1}School of Computer Science and Engineering, Southeast University, Nanjing, China \\
\textsuperscript{2}Key Laboratory of Computing Power Network and Information Security, Ministry of Education,  \\
Shandong Computer Science Center (National Supercomputer Center in Jinan), Qilu University of \\
Technology (Shandong Academy of Sciences), Jinan, China \\
\textsuperscript{3}Shandong Provincial Key Laboratory of Computing Power Internet and Service Computing, \\
Shandong Fundamental
Research Center for Computer Science, Jinan, China \\
\textsuperscript{4}Key Laboratory of New Generation Artificial Intelligence Technology and Its \\
Interdisciplinary Applications (Southeast University), Ministry of Education, China \\
\textsuperscript{5}School of Computing Information Sciences, Saint Francis University, Hong Kong, China \\
\texttt{\{yuan-li,yaxin,230259148,yhjia\}@seu.edu.cn}, \\
\texttt{gaoyb@qlu.edu.cn}
}

\newtheoremstyle{mytheorem}
  {3pt}{3pt}
  {\itshape}
  {}
  {\bfseries}
  {.}
  { }
  {}

\theoremstyle{mytheorem}

\newtheorem{lemma}{Lemma}
\newtheorem{theorem}{Theorem}

\iclrfinalcopy 
\begin{document}

\maketitle

\begingroup
\renewcommand{\thefootnote}{\fnsymbol{footnote}}
\footnotetext[1]{These authors contributed equally to this work.}
\footnotetext[2]{Corresponding author.}
\endgroup

\begin{abstract}

Semi-supervised learning typically assumes that labeled and unlabeled data share an identical class distribution and label space. However, this setting is often violated: unlabeled data may be imbalanced and contain unknown class samples, causing mismatches in both class distribution and label space. Such dual mismatch leads to majority classes dominating the latent space and unknown class samples being overconfidently misclassified, degrading feature discriminability and pseudo-label quality. To address this, we propose a hub-spoke latent geometry, where known classes are uniformly distributed around a central hub and each class forms compact clusters around its prototype, while the hub provides an anchor for a low-evidence region specifically designed for high-uncertainty unknown class samples. Integrated with an evidence-based classifier, this geometry ultimately enhances feature discriminability and uncertainty separation by mitigating majority-class domination through structured feature organization and guiding high-uncertainty unknown class samples toward the hub. Extensive experiments show that our method outperforms state-of-the-art methods, with a maximum improvement of \textbf{3.25\%} across various settings. 

\end{abstract}
  
\section{Introduction}
\label{sec:intro}

Semi-supervised learning (SSL) effectively leverages large amounts of unlabeled data to improve model performance when labeled data is scarce. Representative SSL techniques, such as pseudo-labeling~\citep{Tri-net, CL}, consistency regularization~\citep{VAT, DS}, or a combination of both~\citep{CGMatch, SimPro} have demonstrated considerable success by utilizing the information within unlabeled data. However, these methods typically rely on the assumption that the labeled and unlabeled data share both a consistent class distribution (i.e., both are balanced) and an identical label space (i.e., no unknown class samples exist in the unlabeled data). \textbf{In practice, however, the unlabeled data often exhibits mismatches in both class distribution and label space} (i.e., the unlabeled data is imbalanced and contains unknown class samples), as illustrated in the rightmost panel of Fig.~\ref{fig:setting}. Existing SSL methods lack an effective strategy to handle these realistic and challenging settings and therefore suffer from significant performance degradation.

\begin{figure}[t]
\centering
\includegraphics[width=1.0\textwidth]{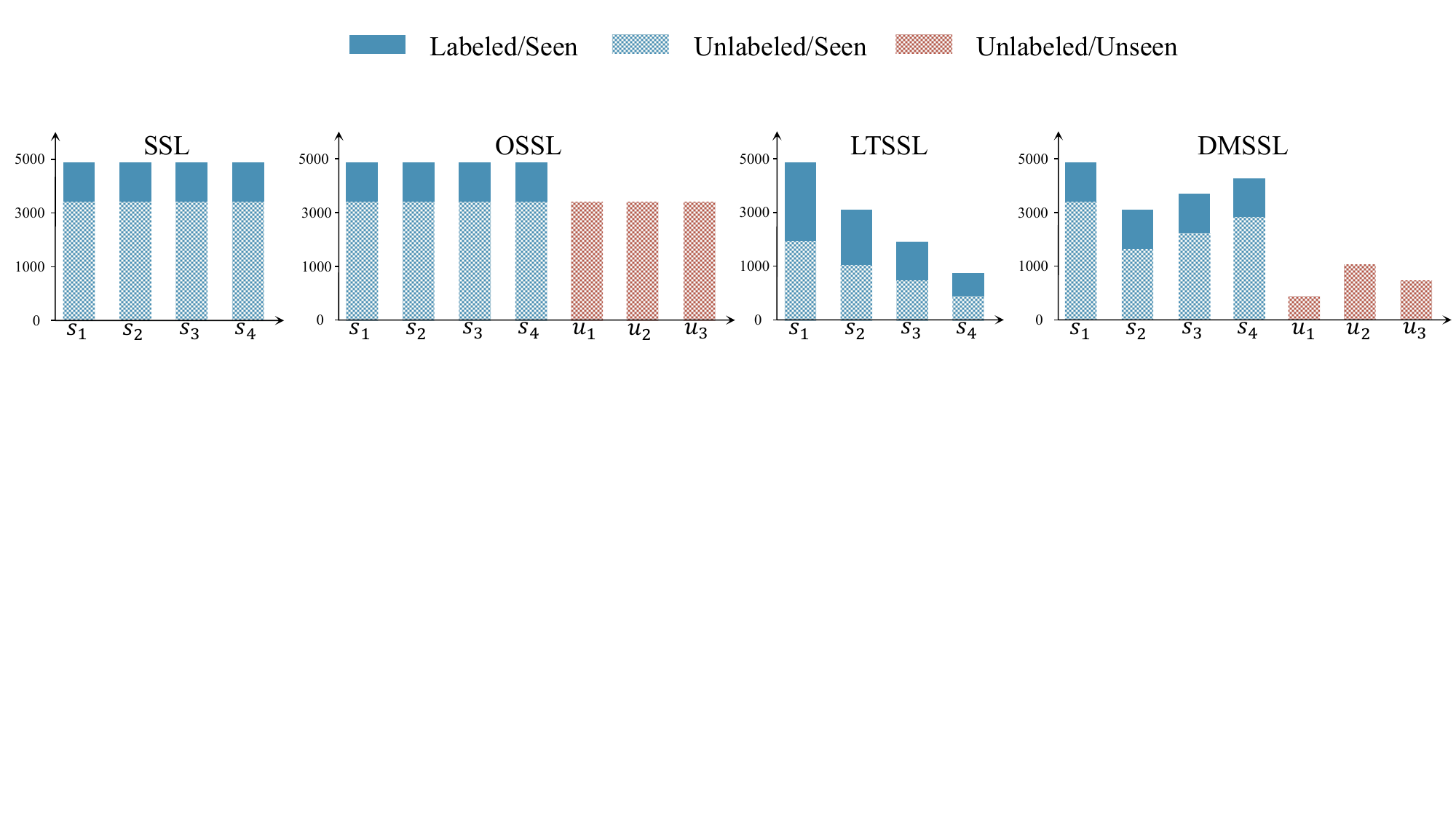}
\vspace{-4.7cm}
\caption{A comparison of the class distributions in semi-supervised learning (SSL), open-set semi-supervised learning (OSSL), long-tailed semi-supervised learning (LTSSL), and the dual-mismatched semi-supervised learning (DMSSL) scenarios. DMSSL considers a more realistic scenario where the labeled data is balanced, while the unlabeled data is imbalanced and contains samples from both known and unknown classes.}
\label{fig:setting}
\vspace{-0.3cm}
\end{figure}

 While existing long-tailed semi-supervised learning (LTSSL) methods tackle class distribution mismatch through approaches like auxiliary classifiers~\citep{Hleash, CPE, Meta-Expert, SAGE} and distribution alignment~\citep{DARP, FlexDA}, they often overlook the underlying feature distribution. We empirically observe that imbalanced class distribution causes majority classes (e.g., \textcolor[HTML]{8C564B}{class 5}) to dominate the feature representation space. This dominance compresses the representation space of minority classes (e.g., \textcolor[HTML]{D62728}{class 3} and \textcolor[HTML]{7F7F7F}{class 7}), leading to overlapping class boundaries (Fig.~\ref{fig:motivation_tsne} left). This degradation limits the model's discriminative capability and leads to low-quality pseudo-labels. Meanwhile, to handle label space mismatch, existing methods rely heavily on softmax-based confidence scores to filter out unknown class samples to prevent negative interference~\citep{OpenMatch, MTCF} or leverage them positively~\citep{IOMatch, CaliMatch}. However, such scores lead to overconfident misclassifications of unknown class samples as known class ones, thereby degrading pseudo-label accuracy (Fig.~\ref{fig:pseudo_label}).

The aforementioned challenges posed by dual mismatch motivate us to explore a  solution that jointly improves known class discriminability and handles unknown classes. To this end, we explicitly organize known and unknown classes in a structured latent space by constructing a uniform hub-spoke geometry (Fig.~\ref{fig:motivation_geometry}), where uniformly distributed spokes represent prototypes for known classes and a central hub provides the anchor for a low-evidence region designed for high-uncertainty unknown classes. By encouraging known classes to spread uniformly around the hub and samples to cluster around their corresponding prototypes (Fig.~\ref{fig:motivation_tsne} right), the geometry promotes inter-class separation and intra-class compactness, mitigating representation bias caused by class imbalance. Building upon this geometry, we further introduce an evidence-based classifier to quantify predictive uncertainty. Specifically, unlabeled samples with high uncertainty which exhibit low similarity to known classes are guided toward the central hub. Our theoretical analysis shows that these high-uncertainty candidate samples, combined with hub-centered clustering, induce a low-evidence region around the hub, providing theoretical support for this uncertainty-guided modeling (see Sec.~\ref{sec:ana} for more details). Consequently, this modeling improves the separation between known and unknown classes in the feature space and their uncertainty distributions, as illustrated in Fig.~\ref{fig:motivation_evidence}.

Our main contributions can be summarized as follows:
\begin{itemize}[leftmargin=22pt]

\item We identify and tackle the realistic dual‑mismatched semi‑supervised learning scenario, where the unlabeled data simultaneously suffers from imbalanced class distribution and the presence of unknown classes.

\item We propose the hub-spoke geometry which enhances feature discriminability under class imbalance and provides a low-evidence region for unknown class samples, with a theoretical characterization showing how high-uncertainty candidates and hub-centered clustering induce this region.

\item Extensive experiments on CIFAR-10, CIFAR-100, SVHN, Food-101, and ImageNet-100 datasets under various dual-mismatched conditions demonstrate that our proposed method achieves new state-of-the-art (SOTA) performance gains by up to \textbf{3.25\%}.
\end{itemize}

\begin{figure*}[t]
    \centering
  
    \newcommand{\picrowwidth}{0.85\textwidth}

    \begin{minipage}{\picrowwidth}
        \centering
        \begin{subfigure}{0.34\linewidth}
            \centering
            \includegraphics[trim=100pt 20pt 150pt 0pt, clip, width=1.0\linewidth]{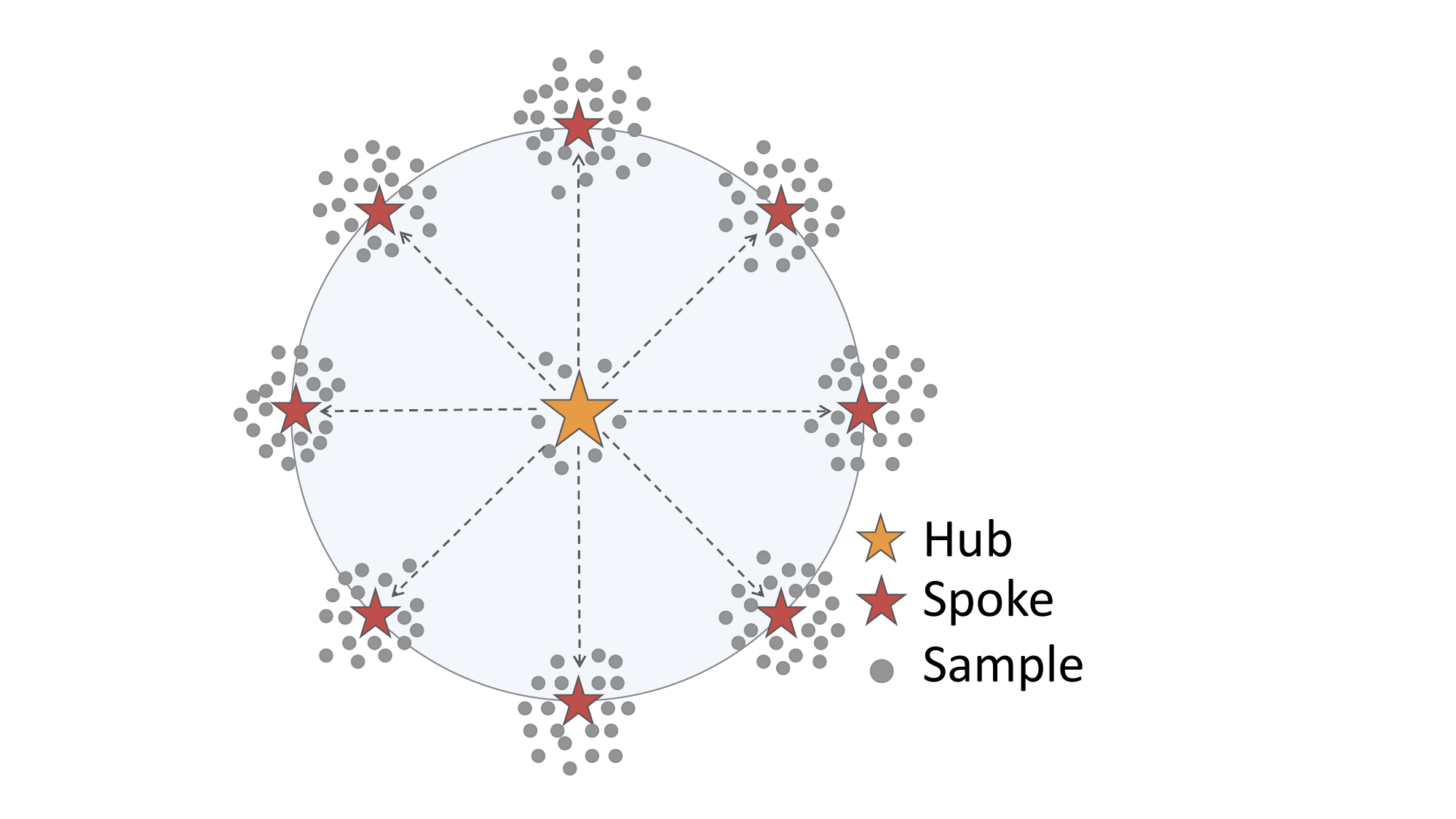}
            \caption{}
            \label{fig:motivation_geometry}
        \end{subfigure}
        \hfill
        \begin{subfigure}{0.63\linewidth} 
            \centering
            \includegraphics[width=0.47\linewidth]{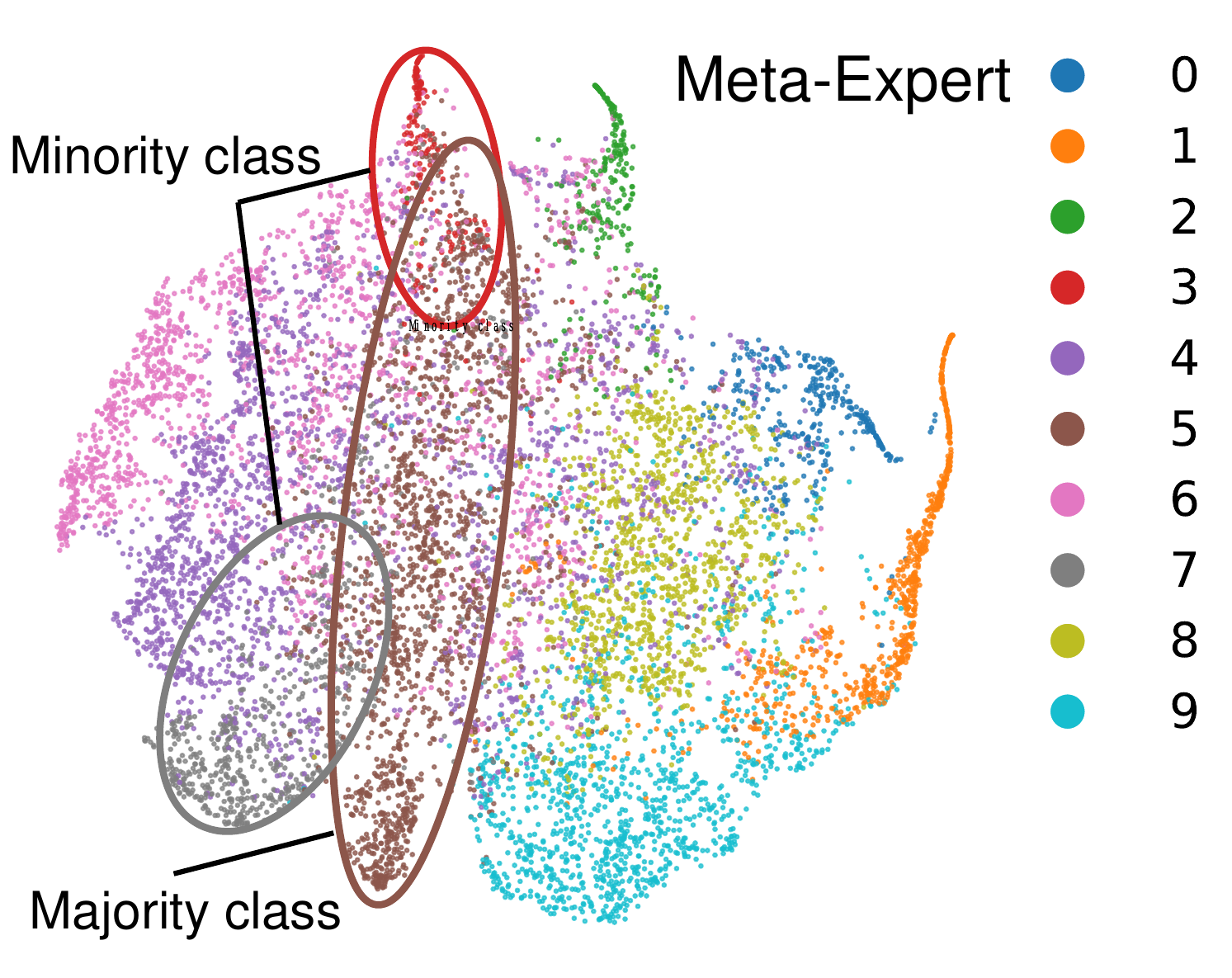}\hfill
            \includegraphics[width=0.47\linewidth]{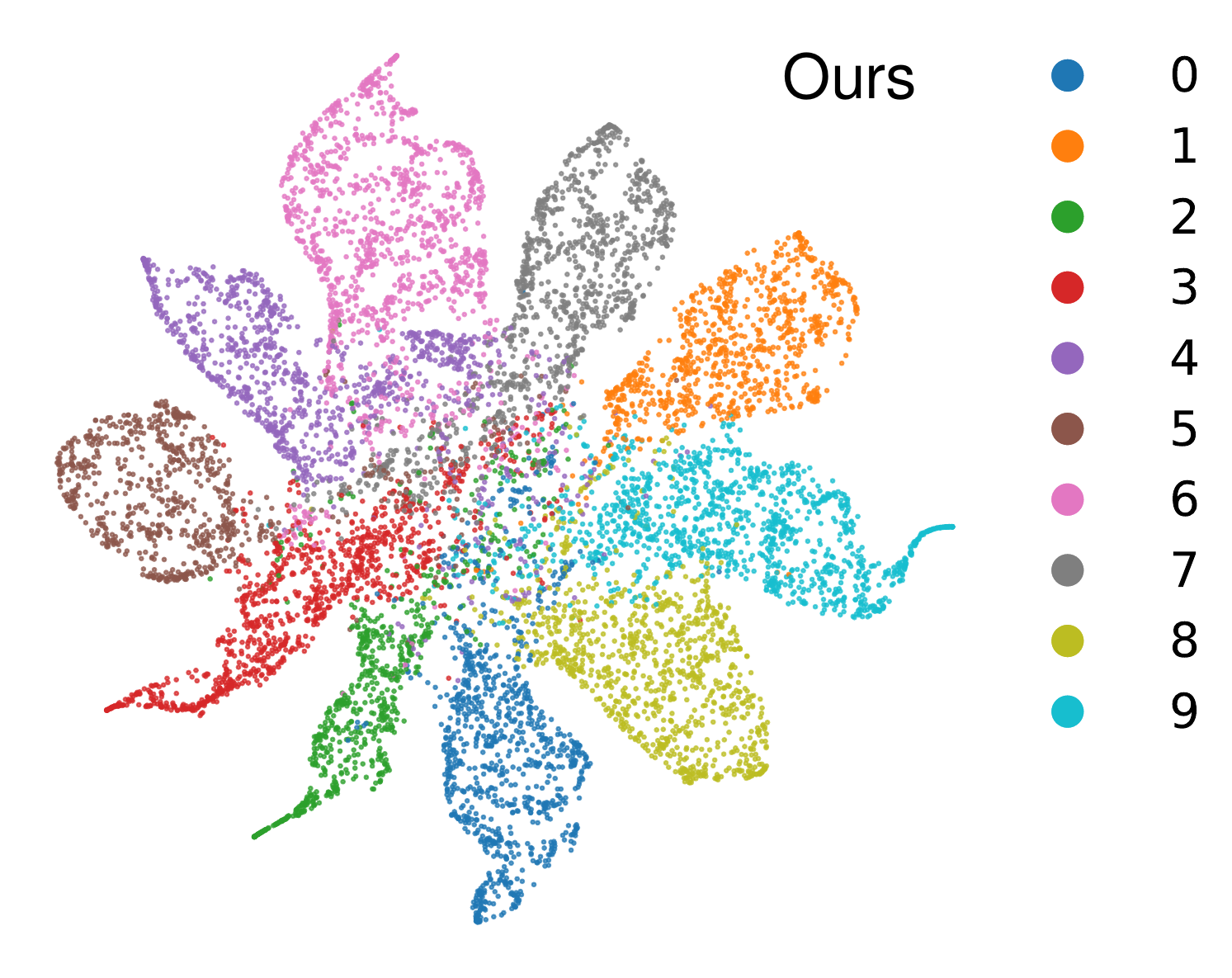}
            \caption{}
            \label{fig:motivation_tsne}
        \end{subfigure}
    \end{minipage}
    
    \vspace{0.2cm}

    \begin{minipage}{\picrowwidth}
        \centering
        \begin{subfigure}{0.34\linewidth}
            \centering
            \includegraphics[width=\linewidth]{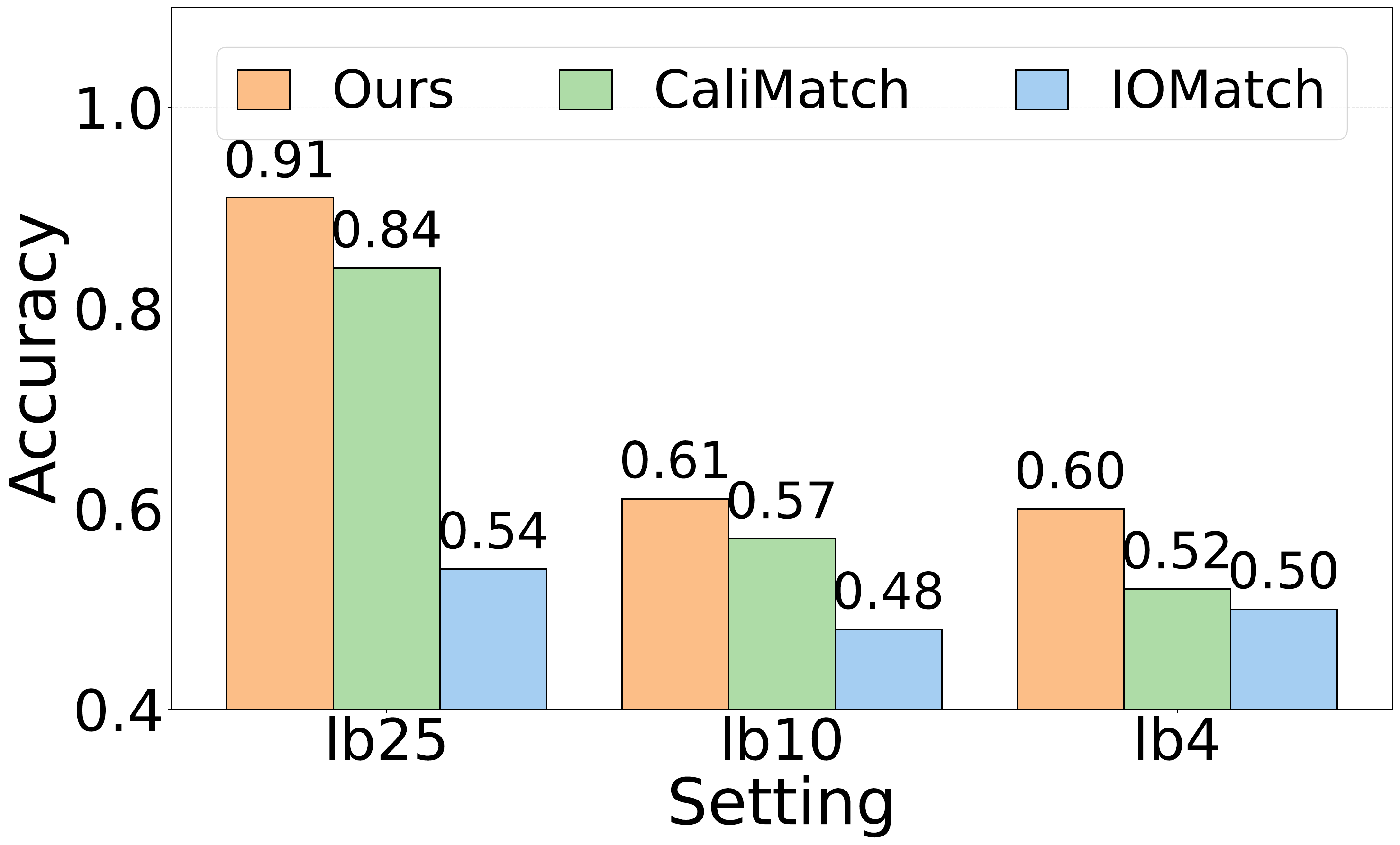}
            \caption{}
            \label{fig:pseudo_label}
        \end{subfigure}
        \hfill
        \begin{subfigure}{0.63\linewidth}
            \centering
            \includegraphics[width=0.47\linewidth]{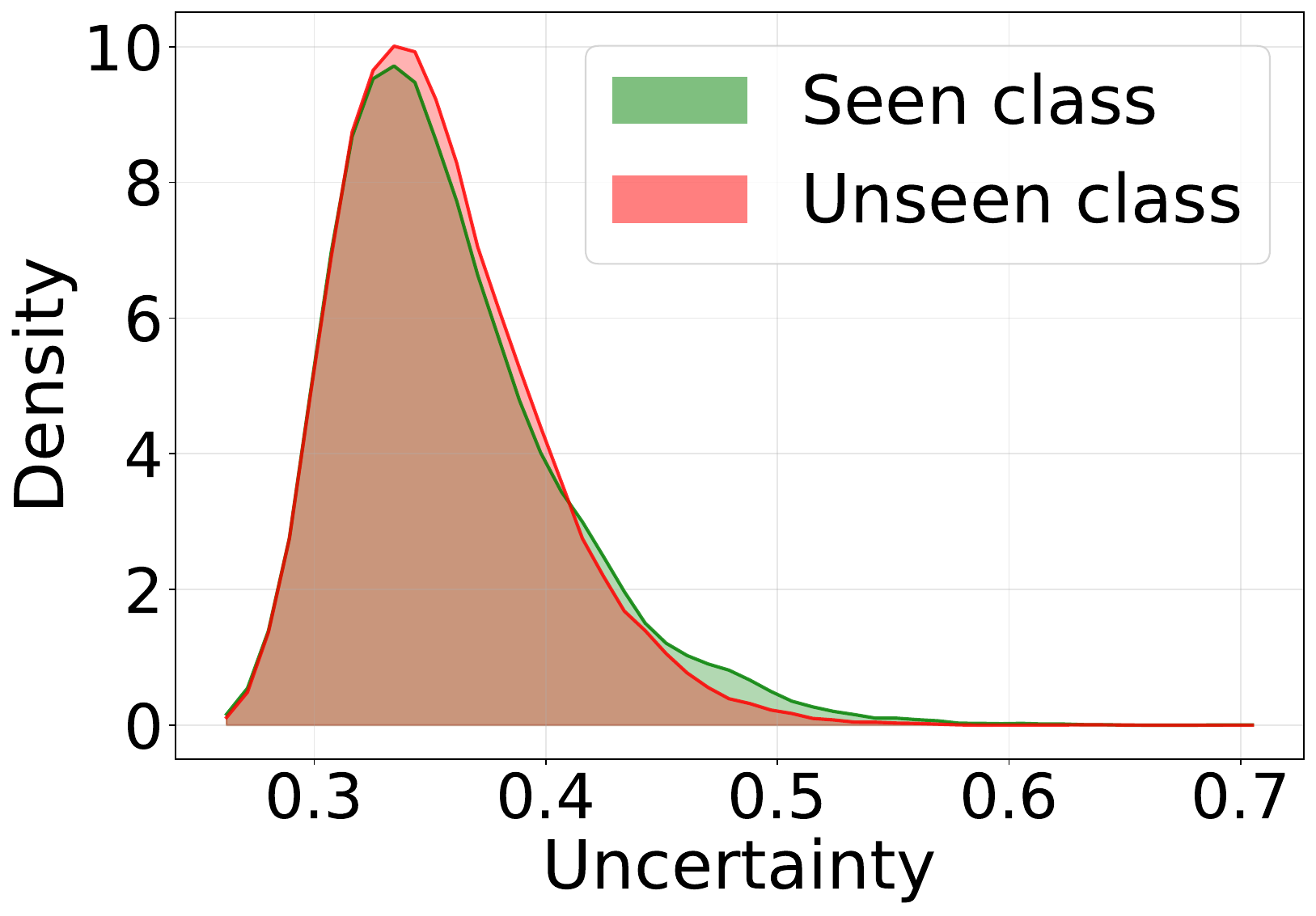}\hfill
            \includegraphics[width=0.49\linewidth]{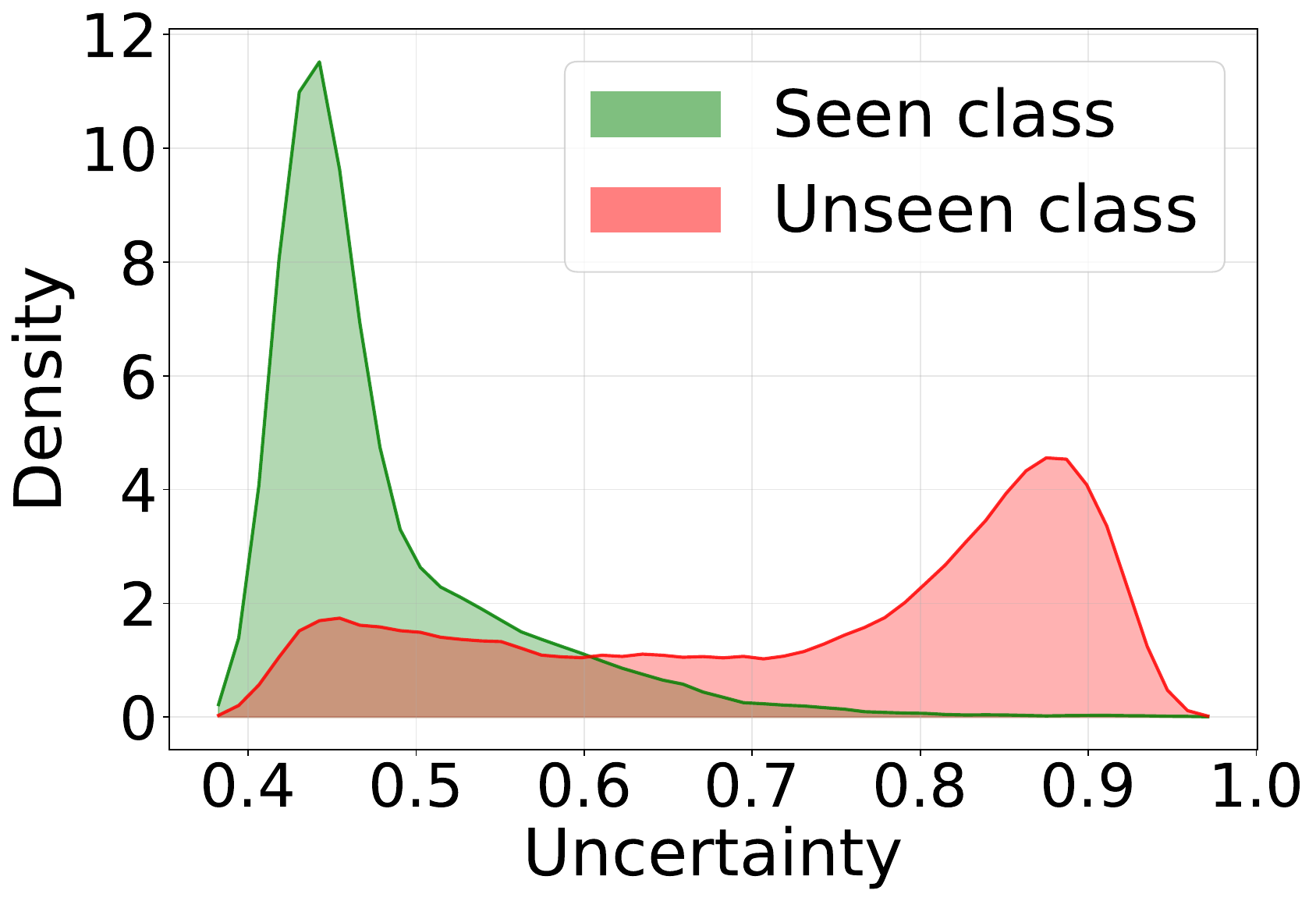}
            \caption{}
            \label{fig:motivation_evidence}
        \end{subfigure}
    \end{minipage}

    \vspace{-0.2cm}
    \caption{(a) presents the proposed hub-spoke geometry; (b) presents the t-SNE visualizations of CIFAR-10~\citep{CIFAR} test set for Meta-Expert~\citep{Meta-Expert} and our proposed method, respectively; (c) shows a comparison of the pseudo-label accuracy on CIFAR-10 among IOMatch~\citep{IOMatch}, CaliMatch~\citep{CaliMatch}, and our method across three settings, with 4 (lb4), 10 (lb10), and 25 (lb25) labeled samples per class, respectively; (d) presents the uncertainty distributions on CIFAR-10 for the cases without guiding and with guiding high‑uncertainty unknown class samples toward the central hub in the hub-spoke geometry, respectively.}
    \label{fig:combined_motivation}
    \vspace{-0.5cm}
\end{figure*}

\section{Related Work}
\label{sec:rel}

\subsection{Long-Tailed Semi-Supervised Learning}

Existing long-tailed semi-supervised learning (LTSSL) methods primarily aim to mitigate model bias toward majority classes during the pseudo-labeling process, and these methods can be broadly categorized into two groups. The first focuses on pseudo-label refinement and distribution alignment. For instance, DARP~\citep{DARP} refines biased pseudo-labels via convex optimization, while DASO~\citep{DASO} leverages the complementary properties of linear and semantic classifiers. Moreover, ACR~\citep{ACR} dynamically calibrates pseudo-labels based on the distance between the estimated unlabeled data distribution and anchor distributions, and CReST~\citep{CReST} employs a progressive distribution alignment strategy to adjust its rebalancing intensity. The second involves structural or strategic adjustments to enforce balance. ABC~\citep{ABC} introduces an auxiliary classifier to alleviate model bias toward majority classes, and CDMAD~\citep{CDMAD} designs a debiasing mechanism by using pattern-free images. More recently, methods such as CPE~\citep{CPE} train multiple experts, each specialized for a specific unlabeled data distribution, to generate higher-quality pseudo-labels. Building on this, Meta-Expert~\citep{Meta-Expert} proposes dynamic expert assignment to better incorporate individual expert strengths. \textbf{Despite their effectiveness in mitigating bias under class imbalance, these methods exhibit limited performance in realistic scenarios where unknown class samples exist in the unlabeled data.}

\subsection{Open-Set Semi-Supervised Learning}

A parallel research direction, known as open-set semi-supervised learning (OSSL), explicitly tackles scenarios where unlabeled data contains unknown class samples. OpenMatch~\citep{OpenMatch} leverages one-vs-all classifiers to isolate unknown class samples from known class ones. IOMatch~\citep{IOMatch} unifies unknown class samples into a new class via a multi-binary classifier, facilitating joint training of closed-set and open-set classifiers. Se-FOSS~\citep{Se-FOSS} utilizes self-supervised learning techniques and energy-based scores to learn from all known and unknown class samples, while SCOMatch~\citep{SCOMatch} explicitly selects reliable unknown class samples to form an additional class and refines the decision boundary through simultaneous closed-set and open-set self-training. Most recently, CaliMatch~\citep{CaliMatch} addresses the critical issue of overconfidence by simultaneously calibrating the classifier and unknown class sample detector. \textbf{However, these methods assume a balanced class distribution and consequently perform poorly when unlabeled data is imbalanced.}

Further discussion and analysis of related work is provided in Appendix~\ref{sec:more_related_work}.

\section{Preliminaries}
\label{sec:prel}

\noindent\textbf{Problem formulation.} This paper studies semi-supervised learning involving mismatches in both class distribution and label space between labeled and unlabeled data, which we term dual-mismatched semi-supervised learning (DMSSL). In DMSSL, we consider both a labeled dataset $\mathcal{D}_l = \{(x_i^l, y_i^l)\}_{i=1}^N$ and an unlabeled dataset $\mathcal{D}_u = \{(x_j^u, y_j^u)\}_{j=1}^M$. Here, $x_i^l$ denotes the $i$-th labeled sample with its ground-truth label $y_i^l \in \{1, \ldots, C\}$, while $x_j^u$ represents the $j$-th unlabeled sample with its inaccessible ground-truth label $y_j^u \in \{1, \ldots, C, \ldots, C_{IO}\}$, where $C_{IO} = C + C_O$ is the total number of classes, including $C$ known and $C_O$ unknown classes. $\mathcal{D}_u$ can be decomposed as $\mathcal{D}_u = \mathcal{D}_I \cup \mathcal{D}_O$, where $\mathcal{D}_I$ and $\mathcal{D}_O$ contain known and unknown class samples, respectively. Furthermore, the labeled data $\mathcal{D}_l$ follow a balanced class distribution, whereas the unlabeled data $\mathcal{D}_u$ follow an imbalanced class distribution. Let $N_c$ denote the number of labeled samples in class $c$, defining the labeled data imbalance ratio as $\gamma_l = \frac{\max_c N_c}{\min_c N_c} = 1$. Theoretically, we can assume that $M_c$ represents the number of samples for class $c$ in $\mathcal{D}_u$, and define its imbalance ratio as $\gamma_u = \frac{\max_c M_c}{\min_c M_c} > 1$.

\noindent\textbf{Base semi-supervised learning framework.} Common semi-supervised learning methods can be summarized as follows. Given a batch of $B_l$ labeled samples $\mathcal{X}^l = \{(x_i^l, y_i^l)\}_{i=1}^{B_l}$, a weak augmentation $\Omega_w(\cdot)$ is applied to each sample before being fed to an encoder $\Phi(\cdot)$ that extracts feature representations $z = \Phi(\Omega_w(x))$, which are mapped to logits by a fully-connected layer $f(\cdot)$ to produce class probability distributions $p = f(z)$. The model's predictions on these labeled samples are supervised by the standard cross-entropy loss $\mathcal{H}(\cdot)$ against their ground-truth labels:
\begin{equation}
\mathcal{L}_{sup} = \frac{1}{B_l} \sum_{i=1}^{B_l} \mathcal{H}(p_i, y_i).
\label{eq:sup}
\end{equation}

Given a batch of $B_u$ samples $\mathcal{X}^u = \{x_j^u\}_{j=1}^{B_u}$, each sample is transformed by both a weak augmentation $\Omega_w(\cdot)$ and a strong augmentation $\Omega_s(\cdot)$, and the resulting two views are then processed to yield class probability distributions $p^w$ and $p^s$, respectively. The unsupervised loss is then defined as the cross-entropy loss $\mathcal{H}(\cdot)$ between the prediction from the strong augmentation view and the pseudo-label generated from the weak augmentation view:
\begin{equation}
\mathcal{L}_{unsup} = \frac{1}{B_u} \sum_{j=1}^{B_u} \mathbb{I} (\max (p_j^w) > \tau) \mathcal{H}(p_j^s, \hat{p}_j^w),
\label{eq:unsup}
\end{equation}
where $\hat{p}_j^w$ is the pseudo-label for the $j$-th unlabeled sample in the current batch, $\tau$ is the confidence threshold, and $\mathbb{I}(\cdot)$ denotes a binary sample mask that selects pseudo-labels with confidence larger than $\tau$. This ensures that only high-confidence pseudo-labels contribute to the training process.
\section{Proposed Method}
\label{sec:meth}

We address the dual-mismatched semi-supervised learning scenario where unlabeled data suffers from both class imbalance and the presence of unknown class samples. Our method constructs a discriminative feature representation space via a hub-spoke geometry. Meanwhile, it identifies high-uncertainty unknown class samples by leveraging evidence-based uncertainty and models them by guiding their features toward the central hub. 

\subsection{Hub-spoke Geometry Construction}

As illustrated in Fig.~\ref{fig:motivation_tsne} left, we observe that imbalanced class distribution leads to majority classes (e.g., \textcolor[HTML]{8C564B}{class 5}) dominating the feature representation space and compressing the feature representation space of minority classes (e.g., \textcolor[HTML]{D62728}{class 3} and \textcolor[HTML]{7F7F7F}{class 7}), thereby causing overlapping class boundaries. This limits feature discriminability and leads to low-quality pseudo-labels. To counteract this, we propose both a dual regularizer that structures the feature representation space into a hub-spoke geometry from global and local perspectives, and a class-adaptive re-weighting strategy to dynamically focus on minority classes.

\noindent\textbf{Global structure regularization.} To prevent majority classes from dominating the feature representation space and amplifying model bias, we enforce a globally uniform feature distribution, where all known classes are uniformly distributed around a central hub. We achieve this through two losses designed to explicitly shape the global geometry.

First, an inter-class separation loss $\mathcal{L}_{ig}$ is introduced to directly maximize the angular distances between labeled class prototypes. Formally, $\mathcal{L}_{ig}$ is defined as:
\begin{equation}
\mathcal{L}_{ig} = \frac{1}{|C|(|C|-1)} \sum_{i=1}^{|C|} \sum_{\substack{j=1,\\ j \neq i}}^{|C|} \left(\frac{\mu_i^l \cdot \mu_j^l}{\|\mu_i^l\|_2 \|\mu_j^l\|_2} + 1 \right)^2,
\end{equation}
where $\mu_i^l$ denotes the labeled class prototype for the $i$-th class, $\cdot$ denotes the dot product, and $\|\cdot\|_2$ denotes the $L_2$ norm.

Second, a prototype geometry loss $\mathcal{L}_{og}$ regulates the spatial arrangement of all labeled class prototypes relative to the global labeled class prototype $\bar{\mu}^l = \frac{1}{|C|} \sum_{i=1}^{|C|} \mu_i^l$. Formally, $\mathcal{L}_{og}$ is defined as:
\begin{equation}
\mathcal{L}_{og} = -\log\left(\frac{1}{|C|} \sum_{i=1}^{|C|} d_i \right) + \frac{1}{|C|} \sum_{i=1}^{|C|} (d_i - \bar{d})^2,
\end{equation}
where $d_i = \|\mu_i^l - \bar{\mu}^l\|_2$ is the distance from the $i$-th labeled class prototype to the global labeled class prototype, and $\bar{d} = \frac{1}{|C|} \sum_{i=1}^{|C|} d_i$ is the average distance.

Accordingly, the global structure loss $\mathcal{L}_{g}$ integrates these two loss items and is formulated as:
\begin{equation}
\mathcal{L}_{g} = \mathcal{L}_{ig} + \mathcal{L}_{og}.
\label{eq:global_loss}
\end{equation}

\noindent\textbf{Local structure regularization.} To enforce local compactness and uniformity, we ensure that samples cluster tightly and uniformly around their respective overall class prototypes (i.e., calculated on labeled and pseudo-labeled data). This is achieved through two losses.

The intra-local loss $\mathcal{L}_{il}$ maximizes the similarity between samples and their corresponding overall class prototypes while minimizing similarity to the overall class prototypes of all other classes, enhancing intra-class compactness. Formally, $\mathcal{L}_{il}$ is defined as:
\begin{equation}
\mathcal{L}_{il} = -\frac{1}{B} \sum_{i=1}^B \log\frac{\exp(z_i \cdot \mu_{y_i} / v)}{\sum_{j=1, j \neq y_i}^{C} \exp(z_i \cdot \mu_j / v)},
\label{eq:il_loss}
\end{equation}
where $z_i$ is the feature representation of $x_i$ extracted by the encoder $\Phi(\cdot)$, $\mu_{y_i}$ is the overall class prototype for class $y_i$, $v$ is the temperature parameter, and $\exp(\cdot)$ is the exponential function.

The radial uniformity loss $\mathcal{L}_{ol}$ promotes a uniform feature distribution around each labeled class prototype by minimizing pairwise cosine similarities between normalized feature representation directions within a class, preventing collapse into a narrow region. Formally, $\mathcal{L}_{ol}$ is defined as:
\begin{equation}
\mathcal{L}_{ol} = \frac{1}{|C|} \sum_{c=1}^{|C|} \bigg[\frac{1}{N_c(N_c-1)} \sum_{i=1}^{N_c} \sum_{\substack{j=1,\\ j \neq i}}^{N_c} (r_{c,i} \cdot r_{c,j}) \bigg]^2,
\end{equation}
where $r_{c,i} = (z_{c,i} - \mu_c^l) / \|z_{c,i} - \mu_c^l\|_2$ is the normalized feature representation direction of sample $x_i$ in class $c$, and $z_{c,i}$ is the feature representation of sample $x_i$ belonging to class $c$.

Accordingly, the local structure loss $\mathcal{L}_{l}$ integrates these two loss items and is formulated as:
\begin{equation}
\mathcal{L}_l = \mathcal{L}_{il} + \mathcal{L}_{ol}.
\label{eq:local_loss}
\end{equation}

\begin{wrapfigure}{r}{0.4\linewidth}
\centering
\vspace{-0.0cm}
\includegraphics[width=1\linewidth]{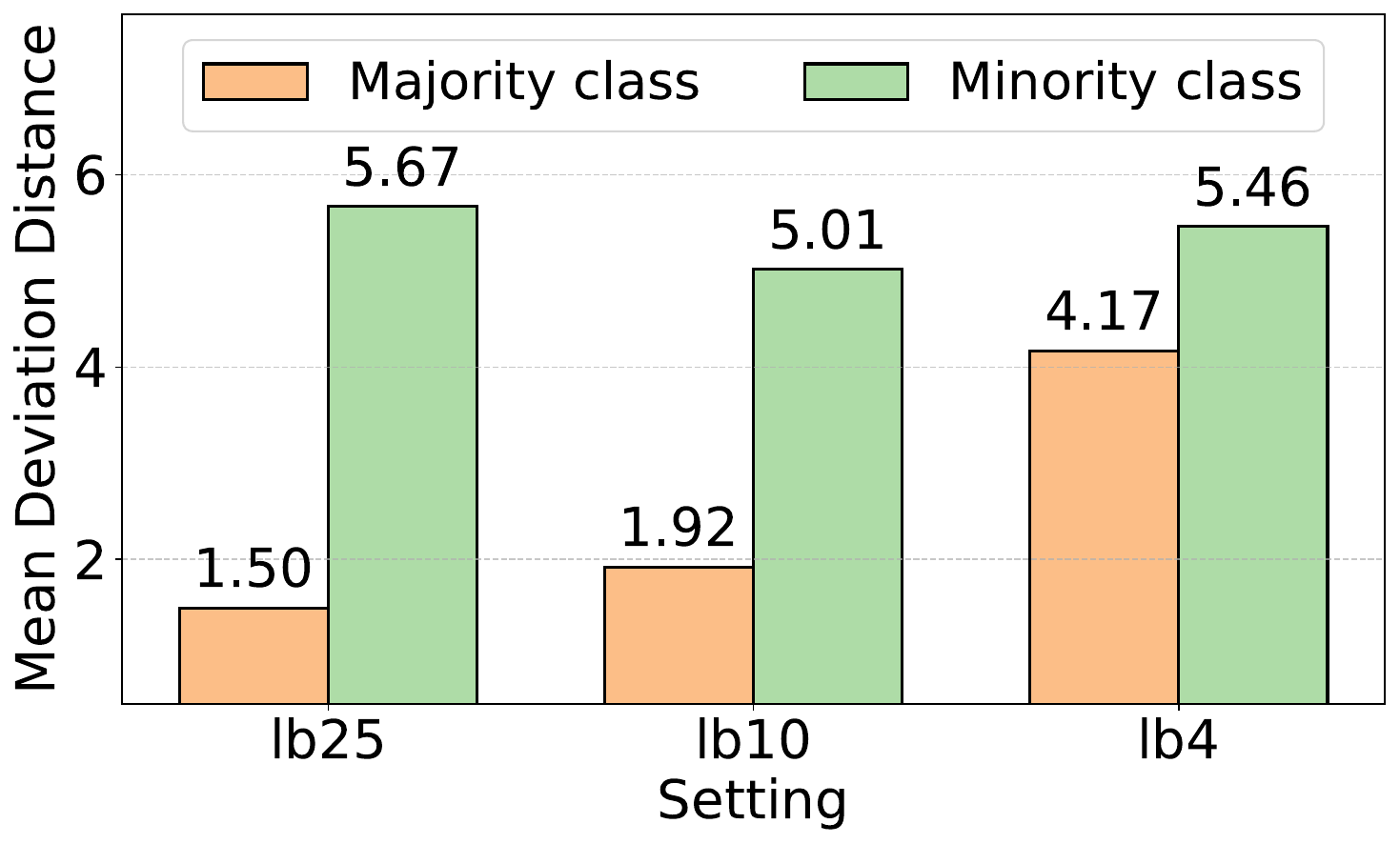}
\vspace{-0.6cm}
\caption{A comparison of the mean deviation distance between the overall class prototype and the labeled class prototype.}
\label{fig:deviation_distance}
\vspace{-0.2cm}
\end{wrapfigure}

\noindent\textbf{Class-adaptive re-weighting.} While the dual regularizer structures the feature representation space into the hub-spoke geometry and improves its discriminability, enhancing the feature discriminability for minority classes remains a critical challenge in long-tailed semi-supervised learning. Fortunately, our empirical analysis reveals that samples from minority classes exhibit a significantly larger deviation between the overall class prototype and the labeled class prototype, as shown in Fig.~\ref{fig:deviation_distance}. This observation motivates us to leverage this deviation for a class-adaptive re-weighting strategy that dynamically prioritizes samples from minority classes. Specifically, for each sample, we compute the cosine similarity between the displacement from the overall class prototype to the sample feature representation (i.e., $q_i = z_i - \mu_{y_i}$) and the overall deviation vector of the class ($\Delta_{y_i} = \mu_{y_i} - \mu_{y_i}^l$). This similarity score is then used to weight the unsupervised loss in Eq.~\ref{eq:unsup}. Formally, the weight $w_i$ for sample $x_i$ is defined as:
\begin{equation}
w_i = 1 + \max\left(0, \frac{q_i \cdot \Delta_{y_i}}{\|q_i\|_2 \cdot \|\Delta_{y_i}\|_2}\right),
\label{eq:w}
\end{equation}
where $\max(\cdot)$ denotes the maximum function. Since the deviation is inherently larger for minority classes, samples belonging to minority classes receive higher weights, thereby adaptively shifting the model's focus toward these classes.

\subsection{Unknown Class Identification and Modeling}
\label{sec:method_unknown}

To effectively identify high‑uncertainty samples from unknown classes, we employ an evidence-based classifier $g(\cdot)$ for uncertainty quantification. Given a sample $x$, the classifier produces a non-negative evidence vector $e$, where each component $e_c$ indicates the support for classifying the sample into class $c$. This vector parameterizes a Dirichlet distribution $\mathrm{Dir}(\pi|\alpha)$ over the class probabilities, with the concentration parameters defined as $\alpha = e + 1$. The predictive uncertainty is derived from the Dirichlet strength $S = \sum_{c=1}^C \alpha_c$, with the uncertainty defined as $u = C / S$. A high value of $u$ corresponds to low overall evidence, suggesting that the sample likely belongs to an unknown class.

\noindent\textbf{Unknown class identification.} During training, we leverage this uncertainty measure to identify potential high‑uncertainty unknown class samples. For each batch, we rank samples by their uncertainty $u$ in descending order and select the top-$k$ most uncertain samples as candidates from unknown classes, where $k$ is set to 10\% of the current batch size (see Appendix~\ref{sec:further_experimental_results_and_analyses} for the hyperparameter sensitivity analysis on $k$). Concretely, let $u_k$ be the $k$-th largest uncertainty in the batch. We then form an evidence-based mask $\mathbb{I}(u \ge u_k)$ to select unknown class sample candidates, thus constructing the identified high‑uncertainty unknown class sample set $\hat{D}_{O,k}$, which serves as the foundation for subsequent unknown class modeling.

To ensure the evidence-based classifier produces well-calibrated uncertainty estimates, we optimize it using an evidence loss $\mathcal{L}_{edl}$ that regularizes the evidence distribution to align with ground-truth labels while maintaining uncertainty awareness:
\begin{equation}
\begin{aligned}
\mathcal{L}_{edl} = \frac{1}{B} \sum_{i=1}^B \bigg[ & \sum_{j=1}^C y_{ij} \left( \psi(S_i) - \psi(\alpha_{ij}) \right) + 
 \lambda_{kl} \cdot \text{KL}(\mathrm{Dir}(\pi_i|\tilde{\alpha}_i) \| \mathrm{Dir}(\pi_i|\mathbf{1})) \bigg],
\end{aligned}
\label{eq:edl}
\end{equation}
where $\psi(\cdot)$ is the digamma function and is defined as the logarithmic derivative of the gamma function, $\alpha_{ij}$ denotes the concentration parameter for class $j$ of sample $i$, and $\lambda_{kl} = \min(0.6, \max(0, \frac{t_{ep}}{10}))$, where $t_{ep}$ denotes the number of epochs since $\mathcal{L}_{edl}$ begins to be optimized during training. The adjusted parameters $\tilde{\alpha}_i = y_i + (\mathbf{1} - y_i) \cdot \alpha_i$ are obtained by removing the non-misleading evidence from the original parameters $\alpha_i$.

\noindent\textbf{Unknown class modeling.} 
Building upon the established hub-spoke geometry and the identified
high-uncertainty unknown class sample set $\hat{D}_{O,k}$, we introduce an
unknown class sample clustering loss $\mathcal{L}_{oc}$ to explicitly model their feature distribution. The global labeled class prototype
$\bar{\mu}^l$, which serves as the hub, is approximately equidistant from all known class prototypes and well separated from the corresponding spokes, making it a natural anchor for unknown class modeling. Without explicit unknown class modeling, unknown class samples may be dispersed among different known class spokes, causing feature and evidence overlap with known class samples. We therefore regularize the representations of the identified high-uncertainty unknown class samples by pulling them toward the hub. Since these samples exhibit low evidence by construction, guiding them toward the hub induces a low-evidence region around the hub, indicating the absence of known-class evidence (see Appendix~\ref{sec:ana} for theoretical analysis). As the hub remains geometrically separated from all known class spokes, this
design not only preserves the discriminability of known classes but also improves the separation between known and unknown classes in the feature space and their uncertainty distributions, as illustrated in Fig.~\ref{fig:motivation_evidence}. Formally, $\mathcal{L}_{oc}$ is defined as:
\begin{equation}
\mathcal{L}_{oc} = \frac{1}{|\hat{D}_{O,k}|} \sum_{i=1}^{|\hat{D}_{O,k}|} \| z_i^{o} - \bar{\mu}^l \|_2^2,
\label{eq:oc}
\end{equation}
where $\hat{D}_{O,k}$ denotes the set of identified high‑uncertainty unknown class samples and $z_i^{o}$ represents the feature representation of the $i$-th identified unknown class sample.

\subsection{End-to-End Training Process}

Our training process is structured into three phases to ensure stable learning. First, a warm-up phase lasting sixty epochs focuses on learning stable initial feature representations using $\mathcal{L}_{warm1}$:
\begin{equation}
\mathcal{L}_{warm1} = \mathcal{L}_{lcon} + \lambda \left( \mathcal{L}_{g} +  \mathcal{L}_{econ}\right),
\label{eq:warm1}
\end{equation}
where $\mathcal{L}_{lcon}$ and $\mathcal{L}_{econ}$ denote the contrastive losses between the weak and strong augmentation views for the classification head $f(\cdot)$ and the evidence head $g(\cdot)$, respectively.

Subsequently, a second warm-up phase lasting twenty epochs integrates the supervised and unsupervised loss to establish baseline performance and stabilize pseudo-labeling, guided by $\mathcal{L}_{warm2}$:
\begin{equation}
\mathcal{L}_{warm2} = \mathcal{L}_{sup} + \mathcal{L}_{unsup} + \lambda \cdot \mathcal{L}_{econ}.
\label{eq:warm2}
\end{equation}

Finally, the main training phase leverages the full objective function to jointly address mismatches in both class distribution and label space:
\begin{equation}
\begin{aligned}
\mathcal{L}_{overall} = \mathcal{L}_{sup} + \mathcal{L}_{unsup} + \lambda \left( \mathcal{L}_{econ} + \mathcal{L}_{g} + \mathcal{L}_{l} + \mathcal{L}_{edl} + \mathcal{L}_{oc} \right).
\end{aligned}
\label{eq:overall}
\end{equation}

Here, $\lambda$ is a balancing hyperparameter set to 0.1 across all three phases. The class-adaptive weight $w_i$ from Eq.~\ref{eq:w} is incorporated into the $\mathcal{L}_{unsup}$ term to adaptively focus on minority classes. Additionally, the confidence-based mask in Eq.~\ref{eq:unsup} is replaced by the combination of the confidence-based mask and the negated evidence-based mask. The pseudo-code is detailed in Algorithm~\ref{alg:algorithm} in Appendix~\ref{sec:pseudo-code}.

\section{Theoretical Analysis}
\label{sec:ana}

We analyze how uncertainty-guided unknown class modeling induces a
low-evidence region around the hub. Let $\bar{\mu}^{l}=
\frac{1}{C}\sum_{c=1}^{C}\mu_c^l$ denote the hub. Following Sec.~\ref{sec:method_unknown}, $\alpha_c(z) = 1+g_c(z), S(z) = C+\sum_{c=1}^{C}g_c(z), u(z)=\frac{C}{S(z)}$, where $g_c(\cdot)$ indicates $e_c$.

\begin{lemma}[Low evidence of selected candidates]
\label{lemma:lemma1}
Let $\widehat{\mathcal D}_{O,k}$ be the top-$k$ most uncertain candidates in a
training batch, and let $u_k$ be the $k$-th largest uncertainty in the batch. Then, for every $z_i^o\in\widehat{\mathcal D}_{O,k}$,
\begin{equation}
g_c(z_i^o)
\le
\varepsilon_{O,k},
\qquad
\varepsilon_{O,k}
:=
C\left(
\frac{1}{u_k}-1
\right),
\quad
c=1,\ldots,C.
\label{eq:lemma1}
\end{equation}
\end{lemma}

\textit{Proof.}
Since $u(z_i^o)\ge u_k$,
\begin{equation}
S(z_i^o)
=
\frac{C}{u(z_i^o)}
\le
\frac{C}{u_k}.
\end{equation}
Hence
\begin{equation}
\sum_{j=1}^C g_j(z_i^o)
\le
C\left(\frac{1}{u_k}-1\right).
\end{equation}
Using $g_c(z_i^o)\ge0$ gives Eq.~\ref{eq:lemma1}.
\hfill $\square$

\begin{theorem}[Low-evidence region around the hub]
Suppose that the evidence function $g_c(\cdot)$ is $L_c$-Lipschitz, i.e., $|g_c({z}_i)-g_c({z}_j)|\le L_c\|{z}_i-{z}_j\|_2$ for any ${z}_i,{z}_j$. For $\mathcal H(r) =
\left\{ z: \|z-\bar{\mu}^{l}\|_2\le r \right\}$, we have
\begin{equation}
\sup_{z\in\mathcal H(r)}
g_c(z)
\le
\varepsilon_{O,k}
+
L_c
\left(
r+\sqrt{L_{oc}}
\right).
\label{eq:Theorem1}
\end{equation}
\end{theorem}

\textit{Proof.}
Let
$n_{O,k}=|\widehat{\mathcal D}_{O,k}|$. From Eq.~\ref{eq:oc}, we have $
L_{oc}
=
\frac{1}{n_{O,k}}
\sum_{i=1}^{n_{O,k}}
\|z_i^o-\bar{\mu}^l\|_2^2$. Since $L_{oc}$ is the average squared distance to the hub, there exists $z_*^o\in\widehat{\mathcal D}_{O,k}$ such that$\|z_*^o-\bar{\mu}^l\|_2^2\le L_{oc}$, and hence
\begin{equation}
\|z_*^o-\bar{\mu}^l\|_2\le\sqrt{L_{oc}}.
\end{equation}

For any $z\in\mathcal H(r)$,
\begin{equation}
\|z-z_*^o\|_2
\le
\|z-\bar{\mu}^l\|_2
+
\|z_*^o-\bar{\mu}^l\|_2
\le
r+\sqrt{L_{oc}}.
\end{equation}
Since $g_c$ is $L_c$-Lipschitz continuous, Lemma~\ref{lemma:lemma1} gives
\begin{equation}
\begin{aligned}
g_c(z) \le g_c(z_*^o) + L_c\|z - z_*^o\|_2  \le \varepsilon_{O,k} + L_c\left(r+\sqrt{L_{oc}}\right).
\end{aligned}
\end{equation}
Taking the supremum over $ z\in\mathcal H(r)$ yields Eq.~\ref{eq:Theorem1}.
\hfill $\square$

Hence, high uncertainty among the selected candidates and strong hub-centered clustering yield a quantitatively
low-evidence hub region and correspondingly high evidential uncertainty.

\section{Experiments}
\label{sec:exp}

\noindent\textbf{Datasets.} We evaluate our method on five benchmarks (i.e., CIFAR-10~\citep{CIFAR}, CIFAR-100~\citep{CIFAR}, SVHN~\citep{SVHN}, Food-101~\citep{Food}, and ImageNet-100~\citep{ImageNet}). To simulate class distribution mismatch, we construct a balanced labeled dataset with different sizes (i.e., 10 and 25 samples per class), while the unlabeled dataset follows an imbalanced class distribution. For label space mismatch, we introduce unknown class samples into the unlabeled data. Concretely, we add 10 unknown classes from ImageNet-127~\citep{ImageNet} to the CIFAR-10 and SVHN datasets, 50 from ImageNet-127 to the CIFAR-100 and Food-101 datasets, and 50 from CIFAR-100 to the ImageNet-100 dataset. 

\noindent\textbf{Baselines.} We compare our method with three SSL algorithms, including FixMatch~\citep{FixMatch}, FreeMatch~\citep{FreeMatch}, and SoftMatch~\citep{SoftMatch}; four LTSSL algorithms, including SimPro~\citep{SimPro}, CDMAD~\citep{CDMAD}, Meta-Expert~\citep{Meta-Expert}, and SC-SSL~\citep{SC-SSL}; and six OSSL algorithms, including OpenMatch~\citep{OpenMatch}, IOMatch~\citep{IOMatch}, SCOMatch~\citep{SCOMatch}, CaliMatch~\citep{CaliMatch}, ANEDL~\citep{ANEDL}, and GGR~\citep{GGR}. Moreover, we use the supervised learning (SL) setting as an upper-bound reference for performance. Implementation details are provided in Appendix~\ref{sec:implementation_details}.

\subsection{Experimental Results and Analyses}
\label{sec:res}

\begin{table*}[htbp]
\centering
\caption{Classification accuracy (\%) on the test data of CIFAR-10, CIFAR-100, SVHN, Food-101, and ImageNet-100 with varying known/unknown classes and labeled set sizes. We report the mean with standard deviation over three runs with different random seeds.}
\label{tab:cifar-10-100-svhn-results}
\scalebox{0.67}{
\setlength{\tabcolsep}{0.75mm}
\begin{tabular}{cccccccccccc}
\toprule
\multicolumn{2}{c}{Dataset} & \multicolumn{2}{c}{CIFAR-10} & \multicolumn{2}{c}{CIFAR-100} & \multicolumn{2}{c}{SVHN} & \multicolumn{2}{c}{Food-101} & \multicolumn{2}{c}{ImageNet-100} \\
\cmidrule(r){1-2} \cmidrule(r){3-4} \cmidrule(r){5-6} \cmidrule(r){7-8} \cmidrule(r){9-10} \cmidrule(r){11-12}
\multicolumn{2}{c}{Known / Unknown} & \multicolumn{2}{c}{10 / 10} & \multicolumn{2}{c}{100 / 50} & \multicolumn{2}{c}{10 / 10} & \multicolumn{2}{c}{101 / 50} & \multicolumn{2}{c}{100 / 50} \\
\cmidrule(r){1-2} \cmidrule(r){3-4} \cmidrule(r){5-6} \cmidrule(r){7-8} \cmidrule(r){9-10} \cmidrule(r){11-12}
\multicolumn{2}{c}{Labels per class}   &    10   &   25   &   10   &   25   &   10   &   25  &   10   &   25   &   10   &   25 \\
\midrule
\multicolumn{1}{c|}{\multirow{2}{*}{\rotatebox[origin=c]{90}{\textcolor{Periwinkle}{\textbf{SL}}}}} 
                      & CE                               & 78.69\mystd{5.10}  & 80.21\mystd{4.85}  & 51.01\mystd{0.73} & 55.74\mystd{0.57}   & 88.93\mystd{1.19}  & 89.18\mystd{1.05} &  39.79\mystd{0.48}  & 42.56\mystd{0.38}  &  44.63\mystd{0.27}  &  46.61\mystd{1.16}  \\
\multicolumn{1}{c|}{} & LA                     & 85.45\mystd{2.64}  & 86.08\mystd{2.47} &  56.20\mystd{0.48} & 59.67\mystd{0.45}   & 92.31\mystd{1.01}  & 92.63\mystd{1.19} &   45.48\mystd{0.37}  & 47.44\mystd{0.20} &   49.61\mystd{0.25} &  51.37\mystd{0.86}  \\
\midrule
\multicolumn{1}{c|}{\multirow{3}{*}{\rotatebox[origin=c]{90}{\textcolor{Periwinkle}{\textbf{SSL}}}}} 
                      & FixMatch        & 69.57\mystd{5.03}  & 76.86\mystd{4.69}  & 39.40\mystd{1.31} & 51.23\mystd{1.29}  & 89.71\mystd{2.60}  & 90.82\mystd{1.02} &  15.95\mystd{0.55}  & 23.23\mystd{0.25}  &  27.44\mystd{0.81}  &  34.95\mystd{0.76}  \\
\multicolumn{1}{c|}{} & FreeMatch       & 64.45\mystd{6.47}  & 75.72\mystd{5.27} &  32.56\mystd{1.62} & 47.52\mystd{1.21} &  89.10\mystd{2.30}  & 89.83\mystd{2.54} &  19.64\mystd{0.34}  & 27.65\mystd{0.20}  &  27.39\mystd{0.45}  & 36.31\mystd{0.97}   \\
\multicolumn{1}{c|}{} & SoftMatch        & 67.97\mystd{5.22}  & 74.56\mystd{2.94} &  34.22\mystd{1.82} & 48.74\mystd{0.46} &  86.34\mystd{1.30}  & 87.44\mystd{2.05} &  19.68\mystd{0.57}  & 27.71\mystd{1.07}  &  27.13\mystd{1.48}   &  36.01\mystd{1.01}  \\
\midrule
\multicolumn{1}{c|}{\multirow{3}{*}{\rotatebox[origin=c]{90}{\textcolor{Periwinkle}{\textbf{LTSSL}}}}} 
                      &  SimPro             & 26.19\mystd{3.62}  & 68.72\mystd{4.87}  & 35.10\mystd{1.24} & 49.68\mystd{1.29} &  59.60\mystd{5.27}  & 83.89\mystd{3.17} &  13.06\mystd{1.23}  & 25.18\mystd{0.45}  &  26.19\mystd{0.46}  &   36.14\mystd{0.39} \\
\multicolumn{1}{c|}{} &  CDMAD               & 58.77\mystd{8.60}  & 70.38\mystd{4.41}  & 24.47\mystd{2.37} & 31.35\mystd{3.08}   & 81.35\mystd{3.20}  & 82.54\mystd{2.09} &  6.35\mystd{1.25}  & 8.31\mystd{2.47}  &  14.52\mystd{2.49}  & 17.72\mystd{1.43}   \\
\multicolumn{1}{c|}{} &  Meta-Expert   & 57.37\mystd{7.89}  & 69.83\mystd{5.37} &  32.05\mystd{1.54} & 47.54\mystd{1.13}  & 79.00\mystd{2.66}  & 86.38\mystd{1.14} &  17.74\mystd{0.39}  & 25.32\mystd{0.59}  &  27.08\mystd{0.37}   &  34.40\mystd{0.34}  \\
\multicolumn{1}{c|}{} &  SC-SSL   & 67.26\mystd{7.13}  & 77.23\mystd{4.41} &  35.23\mystd{1.63} & 49.03\mystd{0.66}  & 89.05\mystd{3.71}  & 89.63\mystd{2.11} &  17.04\mystd{0.23}  & 23.95\mystd{0.12}  &  26.97\mystd{1.18}   &  35.37\mystd{1.0}  \\
\midrule
\multicolumn{1}{c|}{\multirow{4}{*}{\rotatebox[origin=c]{90}{\textcolor{Periwinkle}{\textbf{OSSL}}}}}
                      & OpenMatch        & 47.21\mystd{5.06}  & 68.13\mystd{3.54}  & 32.81\mystd{0.19} & 47.31\mystd{0.44} &  84.88\mystd{5.73}  & 90.47\mystd{0.92} &   10.61\mystd{0.92} & 16.00\mystd{0.39}  &  19.29\mystd{0.83}  &  25.74\mystd{1.34}  \\
\multicolumn{1}{c|}{} & IOMatch            & 67.33\mystd{5.06}  & 77.23\mystd{3.11}  & 38.90\mystd{1.50} & 50.62\mystd{1.12}   & 87.98\mystd{2.87}  & 90.80\mystd{1.50} &  7.62\mystd{0.37}  & 14.46\mystd{0.73}  &   9.04\mystd{6.76}&  21.92\mystd{0.04}   \\
\multicolumn{1}{c|}{} & SCOMatch         & 63.41\mystd{6.06}  & 74.07\mystd{3.41}  & 30.57\mystd{2.90} & 47.42\mystd{1.81}   & 88.52\mystd{2.86}  & 90.74\mystd{0.94} &  13.14\mystd{2.65}  & 16.03\mystd{2.88}  &    17.36\mystd{2.36}&  22.04\mystd{1.31}   \\
\multicolumn{1}{c|}{} & CaliMatch       & 51.33\mystd{10.72} & 67.80\mystd{0.16}  & 35.99\mystd{6.00} & 50.86\mystd{5.36}   & 55.55\mystd{14.84} & 90.71\mystd{1.66} &  7.22\mystd{1.95}  & 10.22\mystd{5.44}  &   15.17\mystd{1.89} &  25.78\mystd{3.46}   \\
\multicolumn{1}{c|}{} & ANEDL       & 27.58\mystd{0.79} & 48.21\mystd{4.01}  & 22.87\mystd{3.54} & 37.26\mystd{3.59}   & 56.30\mystd{1.85} & 76.12\mystd{0.64} &  8.77\mystd{0.69}  & 15.37\mystd{1.87}  &   12.51\mystd{1.56} &  20.23\mystd{0.65}   \\
\multicolumn{1}{c|}{} & GGR       & 61.24\mystd{6.81} & 69.70\mystd{6.69}  & 25.69\mystd{1.43} & 41.67\mystd{2.63}   & 88.21\mystd{0.96} & 88.67\mystd{1.54} &  7.25\mystd{0.66}  & 13.47\mystd{0.98}  &   13.20\mystd{0.48} &  20.35\mystd{0.91}   \\
\midrule
\rowcolor{blue!10} & \textbf{Ours}   & \textbf{71.55\mystd{3.50}} & \textbf{79.20\mystd{2.42}}  & \textbf{42.65\mystd{1.54}} & \textbf{52.99\mystd{0.61}}  & \textbf{90.38\mystd{1.16}} & \textbf{91.28\mystd{0.80}}  &  \textbf{19.82\mystd{0.20}}  & \textbf{27.87\mystd{1.12}}  &  \textbf{27.52\mystd{0.30}}  &  \textbf{36.83\mystd{0.81}}  \\
\bottomrule
\end{tabular}
}
\end{table*}

\noindent\textbf{Main Results.} Table~\ref{tab:cifar-10-100-svhn-results} presents the experimental results on CIFAR-10, CIFAR-100, SVHN, Food-101, and ImageNet-100 datasets under varying known/unknown classes and labeled set sizes. As illustrated in Table~\ref{tab:cifar-10-100-svhn-results}, our method consistently outperforms all baseline methods and is robust to dual mismatch. On CIFAR-10, we outperform the recent method CaliMatch by \textbf{20.22\%} and \textbf{11.40\%} with 10 and 25 labeled samples per class, respectively. Moreover, the performance gain persists on CIFAR-100, where we achieve a 3.25\% performance gain over the best-performing baseline FixMatch, with 10 labeled samples per class. On SVHN, while the performance gap narrows due to the dataset’s lower complexity and clearer known and unknown class distinction, our method still achieves competitive performance.

On the more challenging Food-101 dataset, our method substantially outperforms Meta-Expert by 2.08\% and 2.55\%, respectively, despite the overall low accuracy caused by its fine‑grained and numerous categories. Finally, on the highly complex ImageNet‑100, with 10 labeled samples per class, our method achieves 27.52\%, surpassing CaliMatch by a substantial 12.35\%; with 25 labeled samples per class, we outperform the previous best FreeMatch by 0.52\%. Notably, our method exhibits a lower standard deviation across three runs compared to most baseline methods. Further experimental results and analyses are provided in Appendix~\ref{sec:further_experimental_results_and_analyses}.

\subsection{Ablation Analyses and Discussions}
\label{sec:abla}

\noindent\textbf{The effect of each module.} We conduct comprehensive ablation studies on CIFAR-10 and CIFAR-100 to evaluate the contribution of each proposed module across three settings. As shown in Table~\ref{tab:ablation}, progressively integrating our modules leads to consistent performance gains.

The baseline model, devoid of any proposed modules, exhibits limited capability, especially in the lb4 setting, where it attains only 38.24\% and 26.50\% accuracy on CIFAR-10 and CIFAR-100, respectively. Incorporating the class-adaptive re-weighting (CAW) module brings consistent improvements, with gains of up to 2.10\% on CIFAR-10 and 2.50\% on CIFAR-100. Introducing the local structure regularization (LSR) module yields more substantial improvements, boosting accuracy on CIFAR-10 by \textbf{10.46\%} and CIFAR-100 by 2.35\% in the lb4 setting. 

\begin{wraptable}{r}{0.5\linewidth}
\vspace{-0.4cm}
\caption{Classification accuracy (\%) with and without the key components in the proposed method. }
\label{tab:ablation}
\centering
\setlength{\tabcolsep}{1.0mm}
\begin{adjustbox}{width=0.5\columnwidth, center}
\begin{tabular}{ccccccccccc}
\toprule
\multicolumn{4}{c}{Module} & \multicolumn{3}{c}{CIFAR-10} & \multicolumn{3}{c}{CIFAR-100} \\
\cmidrule(lr){1-4} \cmidrule(lr){5-7} \cmidrule(lr){8-10}
 CAW & LSR & GSR & UCIM & lb4 & lb10 & lb25 & lb4 & lb10 & lb25 \\
\midrule
           &            &            &            & 38.24 & 64.46 & 74.40 & 26.50 & 40.75 & 51.00 \\
\checkmark &            &            &            & 40.34 & 64.86 & 76.00 & 27.89 & 43.25 & 53.00 \\
\checkmark & \checkmark &            &            & 50.80 & 68.05 & 76.95 & 30.24 & 43.25 & 53.49 \\
\checkmark & \checkmark & \checkmark &            & 54.30 & 68.54 & 76.88 & 30.25 & 43.43 & 53.30 \\
\rowcolor{blue!10} \checkmark & \checkmark & \checkmark & \checkmark & \textbf{56.89} & \textbf{70.96} & \textbf{78.05} & \textbf{30.43} & \textbf{44.37} & \textbf{53.53} \\
\bottomrule
\end{tabular}
\end{adjustbox}
\vskip -0.1in
\end{wraptable}
Further adding the global structure regularization (GSR) enhances the local-only variant by an average of 0.65\%. Finally, the unknown class identification and modeling (UCIM) module contributes additional gains, averaging 2.06\% on CIFAR-10 and 0.45\% on CIFAR-100. The full model achieves the best performance, confirming that each module addresses a unique aspect of the challenging dual-mismatch problem involved in the semi-supervised learning scenario and that their effects are complementary and cumulative.

\begin{wrapfigure}{r}{0.5\linewidth}
\vspace{-0.4cm}
\centering
\begin{subfigure}{0.238\textwidth}
\includegraphics[width=\textwidth]{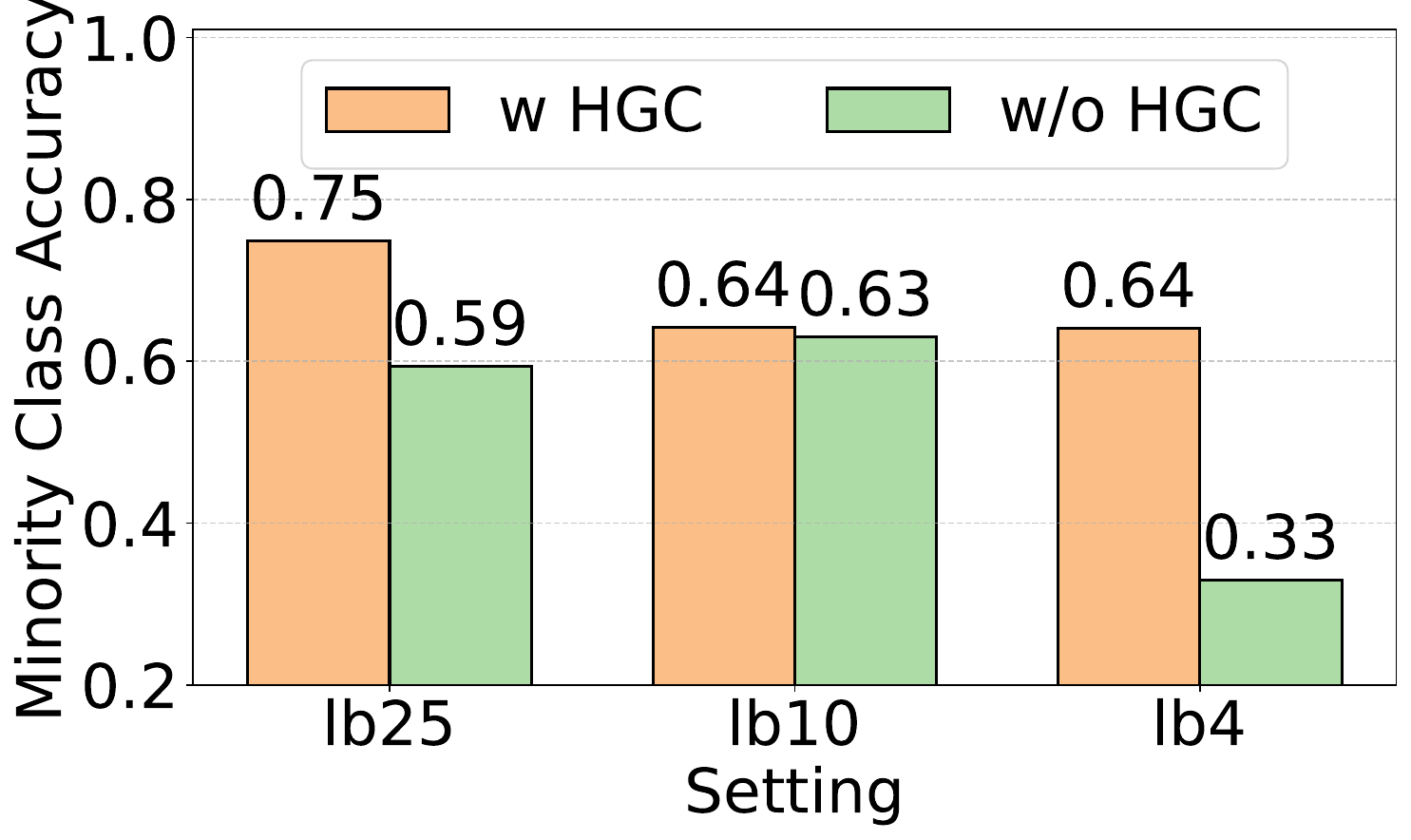}
\caption{}
\label{fig:acc_minority_class}
\end{subfigure}
\begin{subfigure}{0.232\textwidth}
\includegraphics[width=\textwidth]{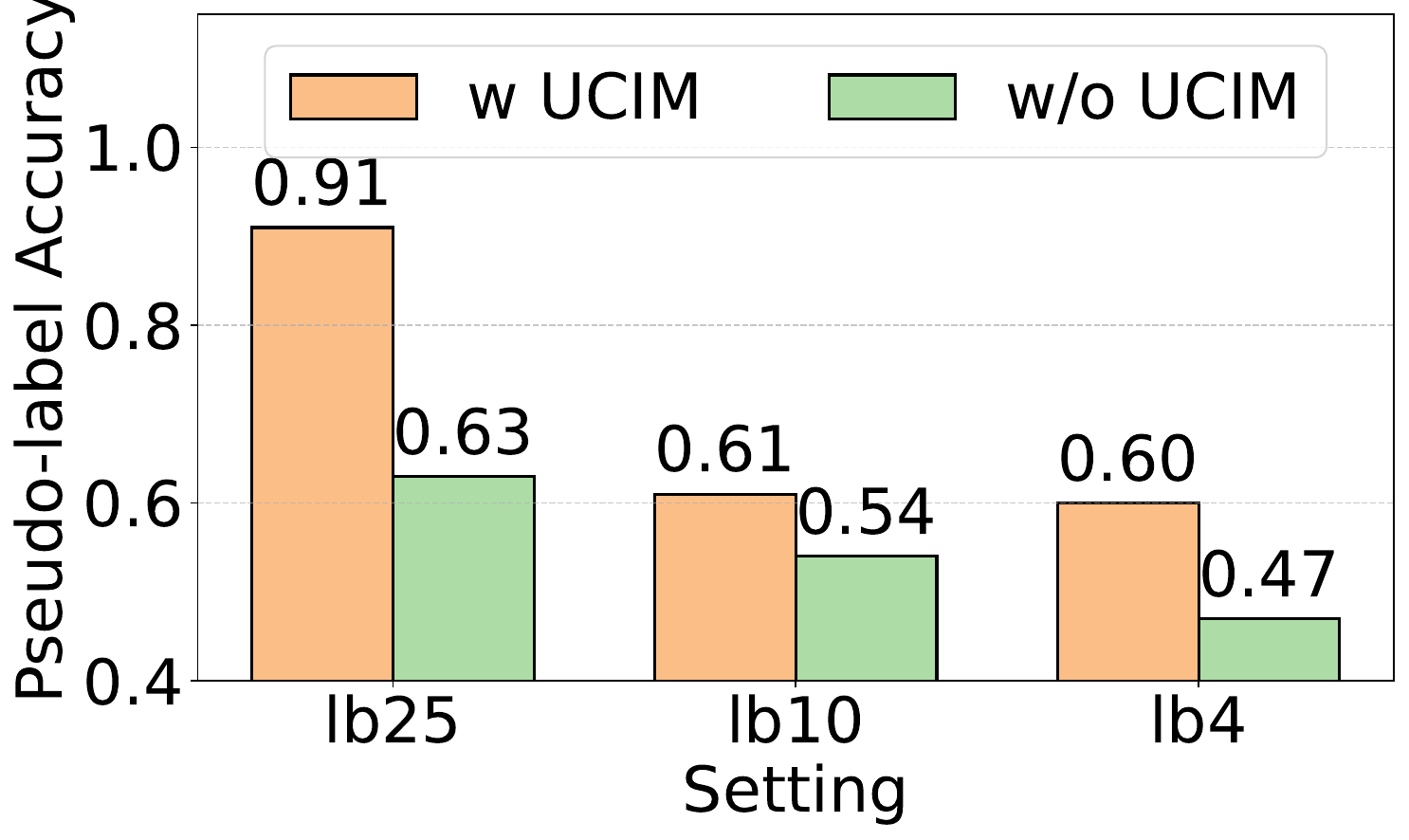}
\caption{}
\label{fig:pesudo_label_ucim}
\end{subfigure}
\vspace{-0.29cm}
\caption{Ablation studies on HGC module for minority class accuracy and UCIM module for pseudo-label accuracy.}
\label{fig:combined_ablation}
\vspace{-0.6cm}
\end{wrapfigure}

Moreover, we evaluate two key modules of our method. The hub-spoke geometry construction (HGC) module which combines LSR and GSR mitigates minority class performance degradation caused by class imbalance. As shown in Fig.~\ref{fig:acc_minority_class}, HGC substantially improves minority class accuracy, e.g., from 33\% to 64\% in the lb4 setting on CIFAR-10, confirming its effectiveness in enhancing feature learning for minority classes under imbalanced distributions. In parallel, the unknown class identification and modeling (UCIM) module boosts pseudo-label accuracy from 47\% to 60\% in the same lb4 setting on CIFAR-10 as illustrated in Fig.~\ref{fig:pesudo_label_ucim}, demonstrating that reliable unknown class handling enhances both pseudo-label quality and robustness against dual mismatch. Further ablation analyses and discussions are provided in Appendix~\ref{sec:further_ablation_analyses_and_discussions}.

\section{Conclusion}

In this work, we addressed the challenging and realistic scenario of dual-mismatched semi-supervised learning, where unlabeled data exhibits an imbalanced class distribution and contains unknown class samples. Our proposed method integrates a structured hub‑spoke geometry with evidence‑based unknown class identification. The hub‑spoke geometry ensures a discriminative and robust feature representation space, while the evidence‑based classifier provides well‑calibrated uncertainty estimates. We further theoretically show that the combination of high-uncertainty samples and hub-centered clustering induces a low-evidence region around the hub, providing theoretical support for our unknown class modeling strategy. Extensive experiments demonstrate that our method consistently outperforms previous state‑of‑the‑art methods under various dual mismatch conditions.

\subsection*{AI use statement}
In this work, we used generative AI tools to assist in the writing of proofs and assist with translation. We have not used generative AI tools for other tasks requiring disclosure. Additionally, we used generative AI tools to  edit software code and the research paper to improve readability. We have reviewed all AI-assisted work. AI-generated proof steps were manually checked for logical correctness. Translations were verified against the source text. AI-edited code was tested for correctness by two authors. All readability edits were approved by all authors. We take responsibility for the final content of this work, including text, claims or artifacts produced with the aid of generative AI.

\subsection*{Ethics statement}
The research utilizes publicly available datasets containing no personally identifiable information. As a strictly technical contribution focused on algorithmic improvement, we foresee no ethical issues or potential
for misuse.
\subsection*{Reproducibility statement}

To ensure the reproducibility of our work, we provide a comprehensive description of the experimental setup in Sec.~\ref{sec:exp} and Appendix~\ref{sec:implementation_details}. All datasets used in our experiments  are publicly available, and the source code is included as supplementary material. These resources collectively enable faithful reproduction of our main experimental results.

\bibliography{iclr2027_conference}
\bibliographystyle{iclr2027_conference}

\appendix
\clearpage
\section*{Table of Contents for the Appendix}
\begin{itemize}[leftmargin=25pt]
    \setlength{\itemsep}{10pt}
    \item ~\ref{sec:more_related_work}: More Related Work \dotfill \pageref{sec:more_related_work}
    
    \item ~\ref{sec:pseudo-code}: Pseudo-code of the Proposed Method \dotfill \pageref{sec:pseudo-code}

    \item ~\ref{sec:implementation_details}: Implementation Details \dotfill \pageref{sec:implementation_details}

    \item ~\ref{sec:further_experimental_results_and_analyses}: Further Experimental Results and Analyses \dotfill \pageref{sec:further_experimental_results_and_analyses}

    \item ~\ref{sec:further_ablation_analyses_and_discussions}: Further Ablation Analyses and Discussions \dotfill \pageref{sec:further_ablation_analyses_and_discussions}

    \item ~\ref{sec:compute_resources}: Compute Resources \dotfill \pageref{sec:compute_resources}

\end{itemize}

\newpage

\section{More Related work}
\label{sec:more_related_work}

Evidential deep learning (EDL), grounded in the Dempster-Shafer theory~\citep{DST} and subjective logic~\citep{SL}, has emerged as a prominent framework for uncertainty quantification and unknown class sample detection. For instance, HEDL~\citep{HEDL} introduces hyper-domain evidence to capture both certain and ambiguous classes, enhancing uncertainty estimation via an opinion projection mechanism, while CEDL~\citep{CEDL} integrates EDL into continual learning, leveraging its ability to distinguish unknown class samples from past known class ones. EDALSSC~\citep{EDALSSC} proposes a semi-supervised evidential model that dynamically balances evidence for active learning sample selection. Notably, ANEDL~\citep{ANEDL} is positioned as the first to adapt EDL for OSSL, utilizing an adaptive negative optimization strategy to improve unknown class sample detection and emphasize uncertain classes. However, it does not address the realistic scenario involving both class distribution and label space mismatches between labeled and unlabeled data.

\section{Pseudo-code of the Proposed Method}
\label{sec:pseudo-code}

\begin{algorithm}[ht]
\caption{Training process of our proposed method}
\label{alg:algorithm}
\begin{algorithmic}[1]
\REQUIRE Labeled and unlabeled datasets $\mathcal{D}_l$ and $\mathcal{D}_u$.
\ENSURE Encoder $\Phi(\cdot)$, classification head $f(\cdot)$, and evidence head $g(\cdot)$.
\STATE Initialize the parameters of $\Phi(\cdot)$, $f(\cdot)$, and $g(\cdot)$ randomly.
\FOR{epoch $= 1, 2, \dots$}
    \FOR{batch $= 1, 2, \dots$}
        \IF {epoch $< 60$}
            \STATE Calculate $\mathcal{L}_{warm1}$ by Eq.~\eqref{eq:warm1};
        \ENDIF
        \IF {$60 \leq$ epoch $< 80$}
            \STATE Calculate $\mathcal{L}_{warm2}$ by Eq.~\eqref{eq:warm2};
        \ENDIF
        \IF {epoch $\geq 80$}
            \STATE Construct hub-spoke geometry by Eq.~\eqref{eq:global_loss} and Eq.~\eqref{eq:local_loss};
            \STATE Compute class-adaptive weight $w$ by Eq.~\eqref{eq:w};
            \STATE Calculate evidence loss $\mathcal{L}_{edl}$ by Eq.~\eqref{eq:edl} for the evidence head;
            \STATE Pull identified unknown class samples toward the hub by Eq.~\eqref{eq:oc};
            \STATE Obtain the overall loss $\mathcal{L}_{overall}$ by Eq.~\eqref{eq:overall};
        \ENDIF
        \STATE Update network parameters via gradient descent;
    \ENDFOR
\ENDFOR
\end{algorithmic}
\end{algorithm}

\section{Implementation Details}
\label{sec:implementation_details}
We follow the default settings and hyperparameters in USB~\citep{USB}. Specifically, the batch size of labeled data $B_l$ is set to 64, while that of unlabeled data $B_u$ is set to 7 times $B_l$; the confidence threshold $\tau$ is set to 0.95, and the temperature parameter $v$ is set to 0.1. Moreover, we use the WRN-28-2~\citep{WRN} architecture, the SGD optimizer with momentum 0.9, and weight decay 5e-4 for training. We use a cosine learning rate decay scheme~\citep{SGDR}, setting the learning rate to $\eta \cos \left(\frac{7 \pi t}{16 T}\right)$. Here, $\eta=0.03$ is the initial learning rate, $t$ is the current training step, and $T=2^{18}$ is the total number of training steps. We repeat each experiment with three different random seeds (i.e., 0, 1, and 4) and report the mean and standard deviation of the performance. We conduct the experiments on a single NVIDIA RTX 4090 GPU using PyTorch v2.4.1.

\section{Further Experimental Results and Analyses}
\label{sec:further_experimental_results_and_analyses}

\subsection{Hyperparameter Sensitivity Analysis}

We perform hyperparameter sensitivity analysis on three key hyperparameters: $k$, the percentage of high-uncertainty samples selected per batch (i.e., the number of identified unknown class samples per batch), $v$, the temperature scaling factor in Eq.~\ref{eq:il_loss}, and $\lambda$, the balancing weight for  losses in Eq.~\ref{eq:overall}. As shown in Fig.~\ref{fig:k_parameter}, our method is robust to the choice of $k$. On CIFAR-10, both $k=10\%$ and $k=20\%$ yield nearly identical performance and significantly outperform the $k=0\%$ baseline across all three settings. Specifically, $k=10\%$ improves over the baseline by +5.0\%, +2.8\%, and +1.0\% for 4, 10, and 25 labeled samples per class, respectively. The comparable results for $k=20\%$ (i.e., gains of +4.7\%, +0.8\%, and +0.1\%) indicate our method's relative insensitivity to this hyperparameter. Similarly, as shown in Fig.~\ref{fig:v_parameter}, performance remains stable across different values of $v$, with an average standard deviation of merely 0.99\%, confirming its robustness. For $\lambda$, we evaluate three values (0.05, 0.10, 0.20) on CIFAR-10, with results reported in Table~\ref{tab:lambda-ablation}. The default $\lambda = 0.10$ achieves the best or near-best accuracy across all labeled settings, e.g., 56.89$\%$ (lb4) and 70.96$\%$ (lb10), while $\lambda = 0.20$ yields a slightly higher result (78.24$\%$) only for lb25. Overall, the performance variation across different $\lambda$ values remains modest, demonstrating that our approach is also robust to this hyperparameter.

\newcommand{\picrowwidth}{0.8\linewidth} 

\begin{figure}[t]
    \centering
    \begin{minipage}{\picrowwidth}
        \centering
        \begin{subfigure}{0.47\linewidth} 
            \includegraphics[width=\linewidth]{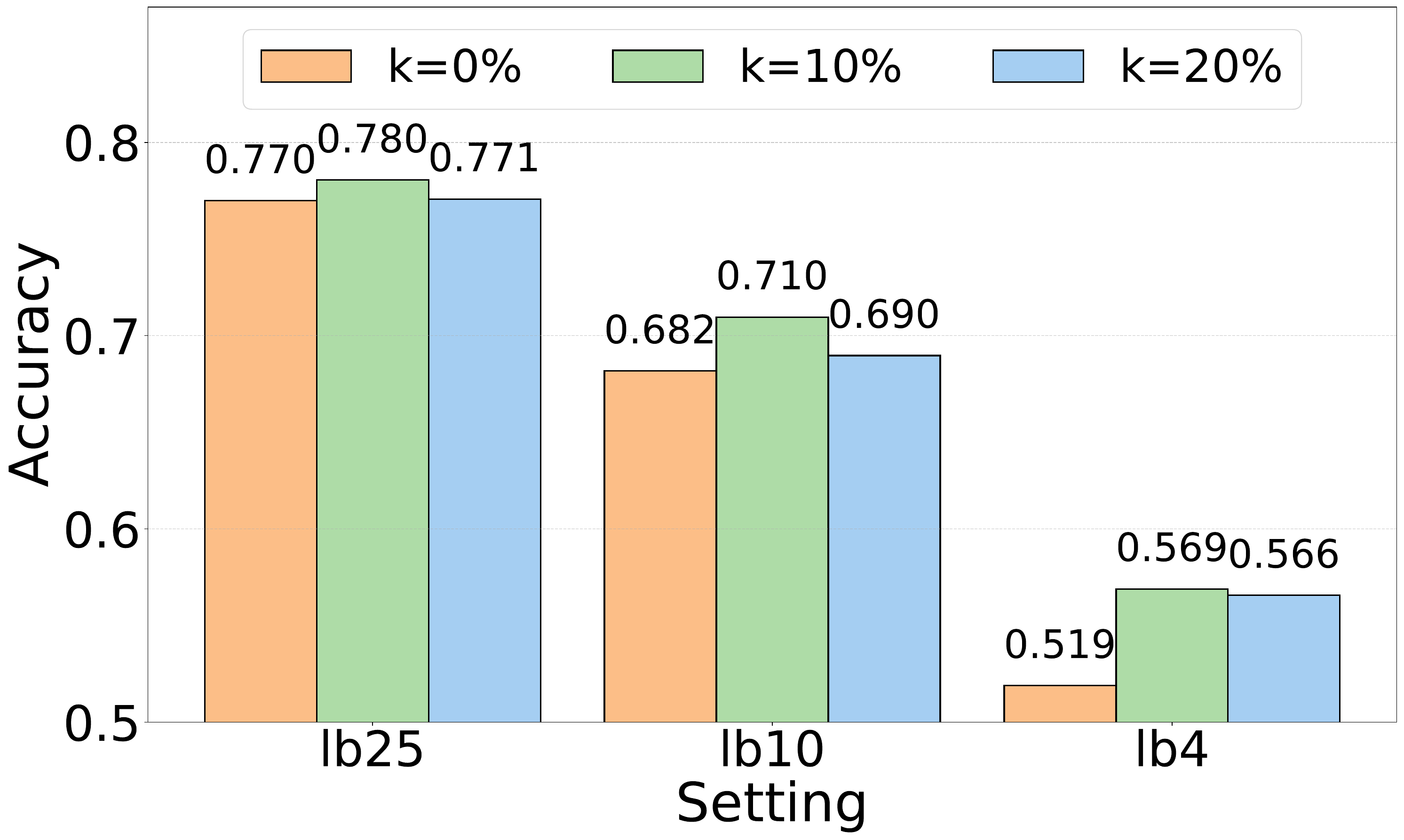}
            \caption{}
            \label{fig:k_parameter}
        \end{subfigure}
        \hfill
        \begin{subfigure}{0.47\linewidth}
            \includegraphics[width=\linewidth]{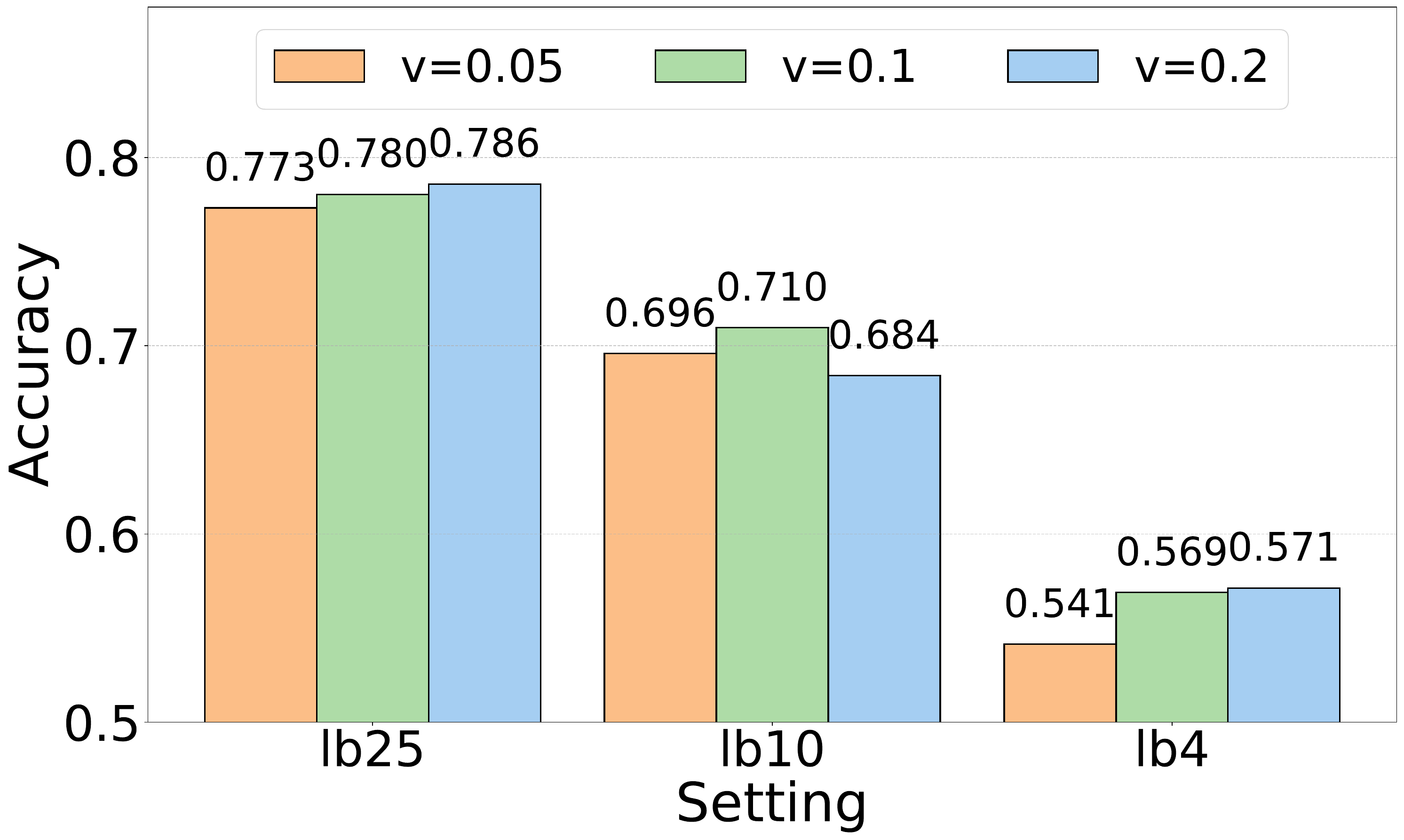}
            \caption{}
            \label{fig:v_parameter}
        \end{subfigure}
    \end{minipage}
    \caption{(a) and (b) respectively compare classification accuracy on the CIFAR-10 test data for \(k \in \{0\%, 10\%, 20\%\}\) and \(v \in \{0.05, 0.1, 0.2\}\) across the three settings. We clearly see that \(k=10\%\) and \(k=20\%\) yield comparable results, and performance remains stable across different values of \(v\). }
    \label{fig:parameters}
\end{figure}

\begin{table}[htbp]
\centering
\setlength{\abovecaptionskip}{1mm}
\caption{Classification accuracy (\%) with different $\lambda$ values on CIFAR-10. The results show that $\lambda = 0.1$ achieves the best or near-best performance across all settings.}
\label{tab:lambda-ablation}
\begin{tabular}{cccc}
\toprule
$\lambda$ & lb4 & lb10 & lb25 \\
\midrule
0.05 & 52.74 & 69.18 & 77.43 \\
\textbf{0.10} & \textbf{56.89} & \textbf{70.96} & 78.05 \\
0.20 & 55.34 & 68.43 & \textbf{78.24} \\
\bottomrule
\end{tabular}
\end{table}

\subsection{Robustness to Unknown Class Diversity and Effectiveness of Unknown Class Modeling}

To further evaluate the robustness and effectiveness of our proposed unknown class handling strategy, we examine the sensitivity of our method to the number of unknown classes by varying it from 50 to 10 on CIFAR-100 while keeping the total unknown class sample count fixed. As shown in Fig.~\ref{fig:acc_vs_unknown_classes}, our method maintains relatively stable performance across all settings despite substantial changes in the diversity of unknown classes. This robustness can be attributed to our hub-based unknown class modeling: instead of explicitly modeling each unknown category, we identify high-uncertainty samples based on their evidence and guide them toward a shared low-evidence region around the hub. Since the hub provides a common target for uncertain samples regardless of their specific unknown class, the same modeling mechanism can be applied as the number of unknown classes changes.

We further compare our unknown class modeling with a K-means clustering baseline on the identified high-uncertainty unknown class samples. For a fair comparison, we set the number of clusters $K$ in K-means to the ground-truth number of unknown classes. As illustrated in Fig.~\ref{fig:kmeans_compare}, our method consistently outperforms K-means across all three settings on CIFAR-10, achieving absolute gains of 0.5\%, 4.8\%, and 1.2\% in the lb4, lb10, and lb25 settings, respectively. The superior performance, particularly in the lb10 setting, highlights the advantage of our hub-spoke geometry, which leverages evidence-based uncertainty to guide high-uncertainty unknown class samples toward a shared low-evidence hub, thereby enhancing known–unknown class discriminability and improving pseudo-label quality.

\begin{figure}[t]
    \centering
    \begin{minipage}{\picrowwidth}
        \centering
        \begin{subfigure}{0.47\linewidth}
            \includegraphics[width=\linewidth]{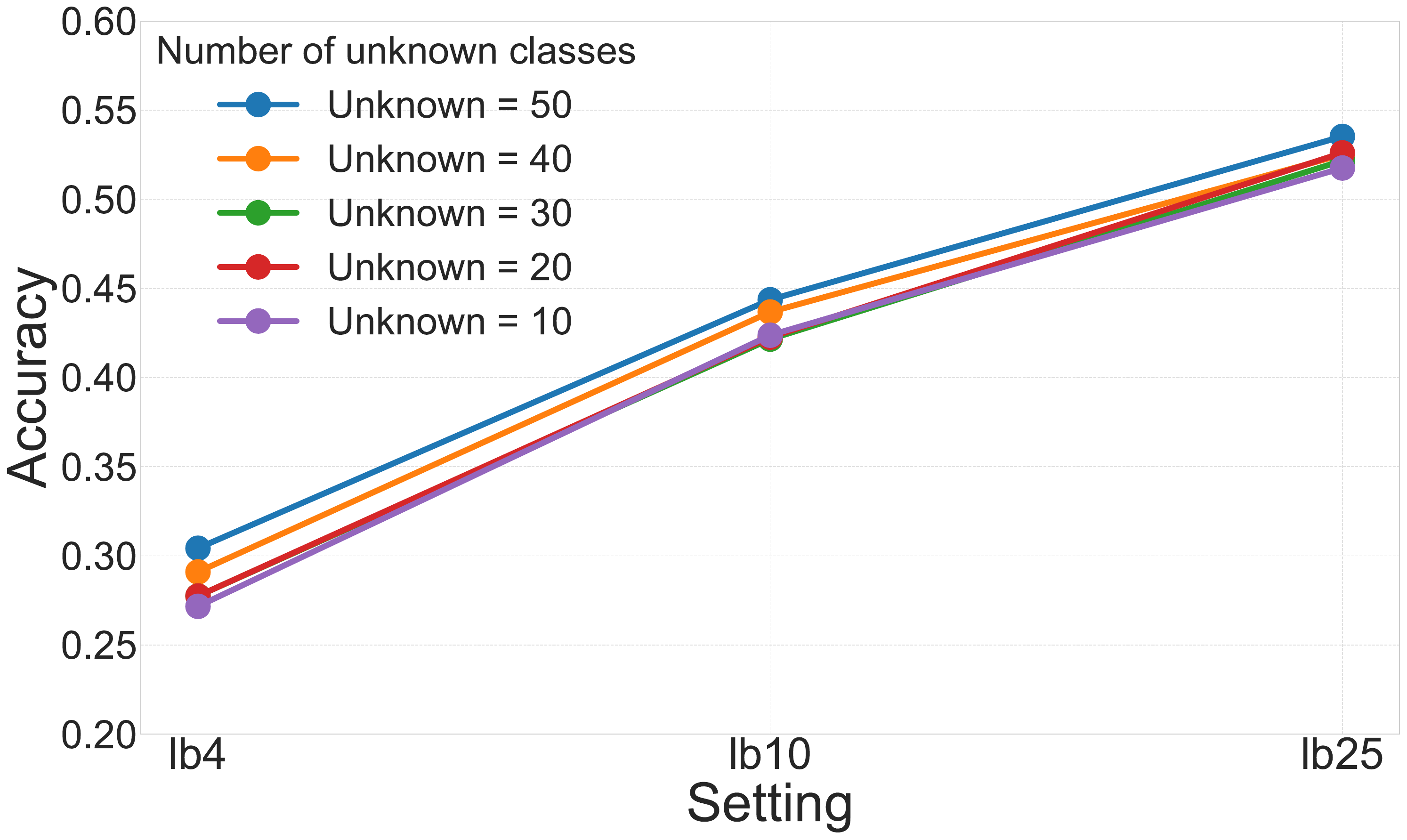}
            \caption{}
            \label{fig:acc_vs_unknown_classes}
        \end{subfigure}
        \hfill
        \begin{subfigure}{0.47\linewidth}
            \includegraphics[width=\linewidth]{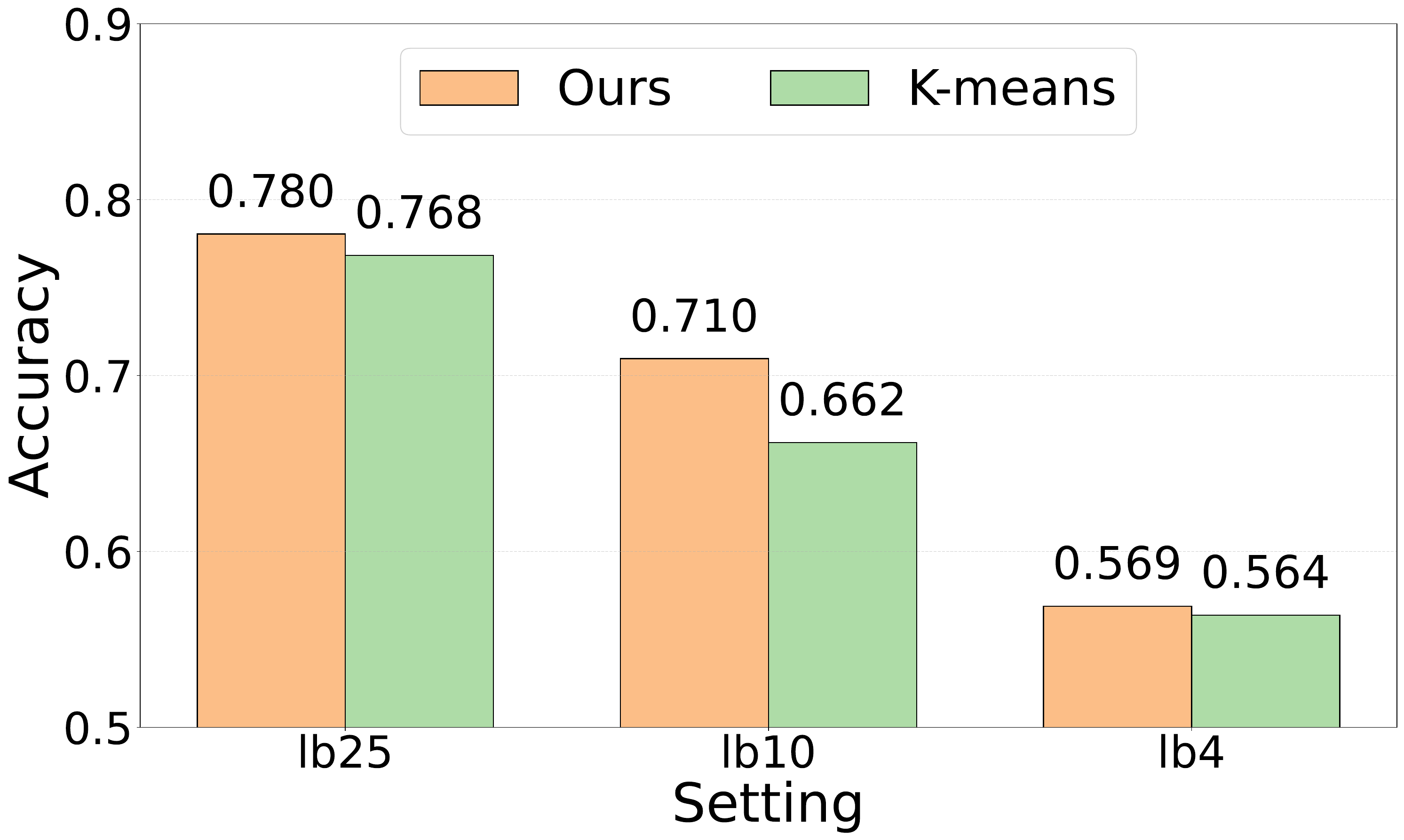}
            \caption{}
            \label{fig:kmeans_compare}
        \end{subfigure}
    \end{minipage}
    \caption{(a) reports classification accuracy on CIFAR-100 under different numbers of unknown classes (50, 40, 30, 20, 10) across the three lb settings; (b) compares our method with K-means clustering on the identified unknown class samples on CIFAR-10.}
    \label{fig:unknown_performance}
\end{figure}

\subsection{Statistical Significance}

We also conduct pairwise t‑tests ($\alpha$ = 0.05) against thirteen baselines under different label settings to evaluate the robustness of our approach and the full results are given in Table~\ref{tab:statistical_significance}. Our method achieves significant improvements over all LTSSL methods in \textbf{75\%} of cases and over all OSSL methods in \textbf{76.7\%} of cases, confirming the consistent superiority of our method.

\begin{table} [htbp]
\caption{Statistical significance of performance differences assessed with pairwise t-test at a 0.05 significance level, reported as win/tie/loss counts.}
\label{tab:statistical_significance}
\centering
\setlength{\tabcolsep}{0.9mm}
\begin{adjustbox}{width=0.38\columnwidth, center} 
\begin{tabular}{cccccc}
\toprule
\multicolumn{2}{c}{Method} & lb10 & lb25 & Total \\
\midrule
\multicolumn{1}{c|}{\multirow{2}{*}{\rotatebox[origin=c]{90}{\textcolor{Periwinkle}{\textbf{SSL}}}}} 
                      & FixMatch         & 2~/~3~/~0 & 1~/~4~/~0 & 3~/~7~/~0 \\
\multicolumn{1}{c|}{} & FreeMatch       & 1~/~4~/~0 & 1~/~4~/~0 & 2~/~8~/~0 \\
\multicolumn{1}{c|}{} & SoftMatch       & 1~/~4~/~0 & 2~/~3~/~0 & 3~/~7~/~0 \\
\midrule
\multicolumn{1}{c|}{\multirow{3}{*}{\rotatebox[origin=c]{90}{\textcolor{Periwinkle}{\textbf{LTSSL}}}}} 
                      &  SimPro            & 5~/~0~/~0 & 3~/~2~/~0 & 8~/~2~/~0 \\
\multicolumn{1}{c|}{} &  CDMAD              & 4~/~1~/~0 & 5~/~0~/~0 & 9~/~1~/~0 \\
\multicolumn{1}{c|}{} &  Meta-Expert  & 4~/~1~/~0 & 4~/~1~/~0 & 8~/~2~/~0 \\
\multicolumn{1}{c|}{} &  SC-SSL  & 2~/~3~/~0 & 3~/~2~/~0 & 5~/~5~/~0 \\
\midrule
\multicolumn{1}{c|}{\multirow{3}{*}{\rotatebox[origin=c]{90}{\textcolor{Periwinkle}{\textbf{OSSL}}}}}
                      & OpenMatch       & 4~/~1~/~0 & 5~/~0~/~0 & 9~/~1~/~0 \\
\multicolumn{1}{c|}{} & IOMatch           & 3~/~2~/~0 & 4~/~1~/~0 & 7~/~3~/~0 \\
\multicolumn{1}{c|}{} & SCOMatch         & 4~/~1~/~0 & 3~/~2~/~0 & 7~/~3~/~0 \\
\multicolumn{1}{c|}{} & CaliMatch       & 2~/~3~/~0 & 3~/~2~/~0 & 5~/~5~/~0 \\
\multicolumn{1}{c|}{} & ANEDL       & 5~/~0~/~0 & 5~/~0~/~0 & 10~/~0~/~0 \\
\multicolumn{1}{c|}{} & GGR       & 4~/~1~/~0 & 4~/~1~/~0 & 8~/~2~/~0 \\
\midrule
\rowcolor{blue!10} & \textbf{Total}  & 41~/~24~/~0 & 43~/~22~/~0 & 84~/~46~/~0 \\
\bottomrule
\end{tabular}
\end{adjustbox}
\end{table}

\section{Further Ablation Analyses and Discussions}
\label{sec:further_ablation_analyses_and_discussions}

To further validate the effectiveness of local structure regularization (LSR) and global structure regularization (GSR), we performed a fine-grained component-level analysis. As shown in Table~\ref{tab:further_ablation}, the integrated LSR module surpasses its individual component losses, namely $\mathcal{L}_{il}$ (LIL) and $\mathcal{L}_{ol}$ (LOL), by average margins of 1.51\% on CIFAR-10 and 0.83\% on CIFAR-100. Similarly, the complete GSR module achieves average improvements of 1.73\% on CIFAR-10 and 0.72\% on CIFAR-100 over its individual component losses $\mathcal{L}_{ig}$ (LIG) and $\mathcal{L}_{og}$ (LOG). These fine-grained ablation results further substantiate the efficacy of both LSR and GSR modules while demonstrating the complementary relationships among their individual component losses.

\begin{table}[htbp]
\centering
\caption{Classification accuracy (\%) with and without the $\mathcal{L}_{il}$ (LIL), $\mathcal{L}_{ol}$ (LOL), $\mathcal{L}_{ig}$ (LIG), and $\mathcal{L}_{og}$ (LOG) losses of the proposed method across three settings, with 4, 10, and 25 labeled samples per class. All ablation studies are conducted with the class-adaptive re-weighting (CAW) module integrated as the foundation. The datasets are CIFAR-10 and CIFAR-100, which comprise a balanced labeled dataset (imbalance ratio $\gamma_l = 1$) of only known classes and an imbalanced unlabeled dataset (imbalance ratio $\gamma_u = 100$) of both known and unknown classes.}
\label{tab:further_ablation}
\scalebox{0.82}{
\setlength{\tabcolsep}{1.25mm}
\begin{tabular}{ccccccccccc}
\toprule
\multicolumn{4}{c}{Module} & \multicolumn{3}{c}{CIFAR-10} & \multicolumn{3}{c}{CIFAR-100} \\
\cmidrule(lr){1-4} \cmidrule(lr){5-7} \cmidrule(lr){8-10}
 LIL & LOL & LIG & LOG & lb4 & lb10 & lb25 & lb4 & lb10 & lb25 \\
\midrule
\checkmark &            &            &            & 50.61 & 67.75 & 76.52 & 29.18 & 43.12 & 52.47 \\
           & \checkmark &            &            & 43.35 & 67.71 & 76.63 & 28.77 & 42.44 & 52.98 \\
\rowcolor{blue!10} \checkmark & \checkmark &            &            & 50.80 & 68.05 & 76.95 & 30.24 & 43.25 & 53.49 \\
\checkmark & \checkmark & \checkmark &            & 52.16 & 67.29 & 76.68 & 29.81 & 42.25 & 52.55 \\
\checkmark & \checkmark &            & \checkmark & 49.33 & 67.26 & 76.34 & 28.65 & 43.20 & 53.20 \\
\rowcolor{blue!10} \checkmark & \checkmark & \checkmark & \checkmark & 54.30 & 68.54 & 76.88 & 30.25 & 43.43 & 53.30 \\
\bottomrule
\end{tabular}}
\end{table}

\section{Compute Resources}
\label{sec:compute_resources}

\begin{itemize}[leftmargin=25pt]
\item CPU: AMD EPYC 7642 48-Core Processor × 2

\item GPU: NVIDIA GeForce RTX 4090 24G × 1

\item MEM: 500G

\item Maximum total computing time: training + testing $\approx$ 20h
\end{itemize}

\typeout{get arXiv to do 4 passes: Label(s) may have changed. Rerun}
\end{document}